\documentclass{article}

\PassOptionsToPackage{numbers,sort&compress}{natbib}
\PassOptionsToPackage{dvipsnames,svgnames}{xcolor}
\usepackage[final]{neurips_2023}

\usepackage[utf8]{inputenc} %
\usepackage[T1]{fontenc}    %

\usepackage{wrapfig}		%

\title{A Dynamical System View of\\Langevin-Based Non-Convex Sampling}

\usepackage[dvipsnames,svgnames]{xcolor}		%

\usepackage{etoolbox}

\usepackage{amsmath,amsfonts,amssymb,bm}
\usepackage{amsthm}
\usepackage{mathrsfs}
\usepackage{tikz}
\usetikzlibrary{positioning,arrows.meta,fit,backgrounds}

\usepackage{mathtools}		%
\usepackage[varg,cmbraces,vvarbb,smallerops]{newtxmath}

\usepackage{acronym}		%
\newcommand{\acdef}[1]{\emph{\acl{#1}} \textup{(\acs{#1})}\acused{#1}}		%

\colorlet{MyRed}{Crimson}
\colorlet{MyGreen}{DarkGreen!80!black}
\colorlet{MyBlue}{MediumBlue}

\newcommand{\afterhead}{.}		%
\newcommand{\para}[1]{\paragraph{\textbf{\S\;#1\afterhead}}}

\usepackage{pifont}		%
\newcommand{\cmark}{\color{MyGreen}\ding{51}}		%
\newcommand{\xmark}{\color{MyRed}\ding{55}}		%

\usepackage{array}		%
\usepackage{booktabs}		%
\usepackage[inline,shortlabels]{enumitem}		%
\setenumerate{itemsep=\smallskipamount,topsep=\smallskipamount,left=\parindent}
\setitemize{itemsep=\smallskipamount,topsep=\smallskipamount,left=\parindent}
\usepackage{makecell}

\usepackage[kerning=true]{microtype}		%

\usepackage{xspace}		%

\usepackage{hyperref}
\hypersetup{
  colorlinks=true,
  linktocpage=true,
  pdfstartview=FitH,
  breaklinks=true,
  pdfpagemode=UseNone,
  pageanchor=true,
  pdfpagemode=UseOutlines,
  plainpages=false,
  bookmarksnumbered,
  bookmarksopen=false,
  bookmarksopenlevel=1,
  hypertexnames=false,
  pdfhighlight=/O,
  urlcolor=MyBlue,linkcolor=MyBlue,citecolor=MyBlue,	%
  pdftitle={},
  pdfauthor={},
  pdfsubject={},
  pdfkeywords={},
  pdfcreator={pdfLaTeX},
  pdfproducer={LaTeX with hyperref}
}

\usepackage[sort&compress,capitalize,nameinlink]{cleveref}		%
\crefname{assumption}{Assumption}{Assumptions}

\usepackage{thmtools}		%
\usepackage{thm-restate}		%

\theoremstyle{plain}
\newtheorem{theorem}{Theorem}		%
\newtheorem*{theorem*}{Theorem}		%
\newtheorem{lemma}{Lemma}		%
\newtheorem{proposition}{Proposition}		%

\newtheorem*{corollary*}{Corollary}		%

\newtheorem{assumption}{Assumption}		%

\theoremstyle{definition}
\newtheorem{definition}{Definition}		%
\newtheorem{example}{Example}		%

\newtheorem*{definition*}{Definition}		%
\newtheorem*{assumption*}{Assumption}		%
\newtheorem*{blanket*}{Blanket assumption}		%
\newtheorem*{example*}{Example}		%

\theoremstyle{remark}
\newtheorem*{remark*}{Remark}		%

\usepackage{twoopt}

\newcommand{\dd}{\:\mathrm{d}}
\newcommand{\dBt}{\dd B_t}
\newcommand{\ddt}{\dd t}

\newcommand{\law}[1]{ \mathrm{law}\prn*{ #1 } }

\newcommand{\Wtwo}{W_2}
\newcommand{\Wsp}{\mathbb{W}_2}

\newcommand{\KLbase}{D_{\mathrm{KL}}}
\newcommand{\KL}[2]{ \KLbase\prn*{ #1 \Vert #2 } }

\newcommand{\Xbar}{\overline{X}}

\newcommand{\gammabar}{\overline{\gamma}}

\newcommand{\Fil}{\mathscr{F}}
\newcommand{\E}{\mathbb{E}}

\newcommand{\noise}{U}
\newcommand{\bias}{b}

\newcommand{\grad}{\nabla}
\newcommand{\stochgrad}{\widetilde{\nabla}}

\newcommand{\qqed}{\ensuremath{\square}}

\newcommand{\rgrad}{\mathop{\mathrm{grad}}}

\newcommand{\Lgv}{L}
\newcommand{\meandynamics}{Langevin flow\@\xspace}

\newcommand{\pert}{Z}
\newcommand{\pertbar}{\overline{Z}}
\newcommand{\step}{\gamma}
\newcommand{\dpt}{x} %
\newcommand{\cpt}{X} %
\newcommand{\stdgauss}{\xi}

\newcommand{\iid}{i.i.d.\@\xspace}

\newcommand{\Xflsymb}{\Phi}
\newcommand{\Xflt}[1]{\Xflsymb^{(#1)}}
\newcommand{\Xfl}[2]{\Xflsymb_{#1}^{(#2)}}

\newcommand{\Xpisymb}{Y}

\newcommand{\Xpi}[2]{\Xpisymb_{#1}^{(#2)}}

\makeatletter
\newcommand\inner{%
  \@ifstar{\innerstar}{\innernostar}%
}
\newcommand{\innernostar}[1]{\ang{#1}}
\newcommand{\innerstar}[1]{\ang*{#1}}
\makeatother

\newcommand{\xmid}[1]{x_{#1 + 1/2}}%

\providecommand\given{}		%

\DeclarePairedDelimiterXPP{\exof}[1]{\E}{[}{]}{}{%
\renewcommand\given{\nonscript\,\delimsize\vert\nonscript\,\mathopen{}} #1}

\DeclarePairedDelimiterXPP{\probof}[1]{\prob}{(}{)}{}{%
\renewcommand\given{\nonscript\:\delimsize\vert\nonscript\:\mathopen{}} #1}

\DeclareMathOperator{\bigoh}{\mathcal{O}}		%

\newcommand{\eps}{\epsilon}
\newcommand{\tr}{\mathrm{tr}}

\newcommand{\N}{\mathbb{N}}		%
\newcommand{\R}{\mathbb{R}}		%
\newcommand{\eg}{e.g.,\xspace}		%
\newcommand{\ie}{i.e.,\xspace}		%
\newcommand{\vs}{vs.\xspace}		%

\DeclarePairedDelimiter{\nrm}{\|}{\|}
\DeclarePairedDelimiter{\abs}{|}{|}
\DeclarePairedDelimiter{\prn}{(}{)}
\DeclarePairedDelimiter{\crl}{\{}{\}}
\DeclarePairedDelimiter{\brk}{[}{]}
\DeclarePairedDelimiter{\ang}{\langle}{\rangle}

\newacro{LRM}{Langevin--Robbins--Monro}
\newacro{APT}{asymptotic pseudotrajectory}
\newacroplural{APT}[APTs]{asymptotic pseudotrajectories}

\newacro{WAPT}[WAPT]{Wasserstein asymptotic pseudotrajectory}
\newacroplural{WAPT}[WAPTs]{Wasserstein asymptotic pseudotrajectories}

\author{%
  Mohammad Reza Karimi\thanks{Equal contribution.} \\
  ETH Z\"{u}rich\\
  \texttt{mkarimi@inf.ethz.ch} \\
\And
  Ya-Ping Hsieh\footnotemark[1] \\
  ETH Z\"{u}rich\\
  \texttt{yaping.hsieh@inf.ethz.ch} \\
\And
  Andreas Krause \\
  ETH Z\"{u}rich\\
  \texttt{krausea@ethz.ch} \\
}

\begin{document}

\maketitle

\begin{abstract}%
  
Non-convex sampling is a key challenge in machine learning, central to non-convex optimization in
deep learning as well as to approximate probabilistic inference. Despite its significance,
theoretically there remain some important challenges: Existing guarantees suffer from the drawback
of lacking guarantees for the \emph{last-iterates}, and little is known beyond the elementary
schemes of stochastic gradient Langevin dynamics. To address these issues, we develop a novel
framework that lifts the above issues by harnessing several tools from the theory of dynamical
systems. Our key result is that, for a large class of state-of-the-art sampling schemes, their
last-iterate convergence in Wasserstein distances can be reduced to the study of their
continuous-time counterparts, which is much better understood. Coupled with standard assumptions of
MCMC sampling, our theory immediately yields the last-iterate Wasserstein convergence of many
advanced sampling schemes such as mirror Langevin, proximal, randomized mid-point, and Runge-Kutta
methods.

\end{abstract}

\section{Introduction}
\label{sec:introduction}

Many modern learning tasks involve sampling from a high-dimensional density $\pi \propto e^{-f}$,
where $f$ is a \emph{non-convex} potential representing, for instance, the loss function of a deep
neural network. 
To this end, an approach that has found wide success is to discretize the continuous-time
\emph{Langevin diffusion}
\begin{equation}
\label{eq:LD}\tag{LD}
  \dd\Lgv_t = -\grad f(\Lgv_t) \ddt + \sqrt{2} \dBt
\end{equation}
where $B_t$ is a Brownian motion \citep{wellingBayesianLearningStochastic2011}. The idea behind
this approach is that, since $\pi$ is the stationary distribution of \eqref{eq:LD}, 
one can expect a similar behavior for discretizations of \eqref{eq:LD}.
Such a framework has inspired numerous sampling schemes with per-iteration costs as cheap as
stochastic gradient descent, which are particularly suitable for large-scale approximate
probabilistic inference and Bayesian learning
\cite{wellingBayesianLearningStochastic2011,ahn2012bayesian,teh2016consistency}.
Moreover, several works have noticed that these Langevin-based schemes provide deep insights about
minimizing $f$ using stochastic oracles \citep{raginsky2017non, erdogdu2018global}, which serves as
an important step toward explaining the empirical success of training deep neural networks.

The convergence of Langevin-based non-convex sampling has therefore attracted significant interest
from both practitioners and theoreticians, whose intense study has led to a plethora of new
guarantees; see related work for details.
Despite such impressive progress, several challenges remain for the fully non-convex setup: 
\begin{itemize}
  \item The convergence is typically given on the \emph{averaged} iterates instead of the more
    natural \emph{last} iterates \citep{teh2016consistency, balasubramanian2022towards}. This is
    especially problematic from the perspective of understanding the minimization of $f$, as in
    practice, the last iterates of an optimization algorithm play the most pivotal role for
    downstream tasks.

  \item An additional notable drawback of the current theory is its predominant focus on the basic
    Euler-Maruyama discretization of \eqref{eq:LD} (see, \eg \cite{teh2016consistency,
    durmus2017nonasymptotic, balasubramanian2022towards}). As a result, the convergence analysis of
    more advanced sampling schemes remains largely unexplored in the fully non-convex regime
    \cite{shen2019randomized,he2020ergodicity,liStochasticRungeKuttaAccelerates2020,hsieh2018mirrored,ahn2021efficient, zhang2020wasserstein}.
\end{itemize}

\para{Contributions and Approaches}

To overcome the aforementioned challenges, our main contribution, from a high level, can be succinctly summarized as:
\begin{equation}
\tag{$\star$}
\label{eq:high-level}
  \parbox{\dimexpr\linewidth-5em}{%
    \strut
    Under mild assumptions, we prove that the iterates of a broad range of Langevin-based sampling schemes converge
    to the \emph{continuous-time} \eqref{eq:LD} in \emph{Wasserstein distance}.%
    \strut
  }
\end{equation}

Combining \eqref{eq:high-level} with classical results on Langevin diffusion
\citep{pavliotis2014stochastic} immediately yields the \emph{last-iterate} convergence in
\emph{Wasserstein distances} for a wide spectrum of sampling schemes, thus resolving all the
challenges mentioned above. To illustrate this point, we state a simple version of our main result.
\begin{theorem*}[Informal]
  Suppose we discretize
  \eqref{eq:LD} as
  \[
    \dpt_{k + 1} = \dpt_k - \step_{k+1}\prn{\grad f(\dpt_k) + \text{noise} + \text{bias}} +
    \sqrt{2\step_{k+1}}\,\xi_{k+1}
  \]
  with step-sizes $\crl{\step_k}_{k\in\N}$ and \iid standard Gaussians $\crl{\xi_{k}}_{k\in\N}$. Then, under an
  easy-to-verify condition on the bias (see \eqref{eq:bias-assumption} in
  \cref{ass:noise-bias}), $\crl{x_k}_{k\in\N}$ converges in Wasserstein distance to $\pi$.
  In addition, these conditions are satisfied by many advanced sampling schemes.
\end{theorem*}
This result is achieved via a new \emph{dynamical perspective} to study
Langevin-based sampling. More specifically, 
\begin{enumerate}
\item We introduce the \emph{Picard process}, which is the sampling analogue of Picard's method of
  successive approximations for solving ODEs \citep{coddington1955theory}. Contrary to most existing
  analyses, the Picard process allows us to completely bypass the use of relative entropy, which is
  the culprit for the appearance of averaged iterates \citep{durmus2017nonasymptotic}.%

\item Using the Picard process, we will prove that the iterates of various Langevin-based schemes
  generate a so-called \acdef{WAPT} for the continuous-time \eqref{eq:LD}. The main motivation for
    considering \ac{WAPT} is to connect Langevin-based schemes to the dynamical system theory of
    \citet{benaim1996asymptotic}, which works for metric spaces and is last-iterate by design, and
    therefore particularly suitable for our purpose. 

\item \looseness-1 Finally, under standard stability assumptions in the literature \citep{roberts1996exponential,
  meyn2012markov}, we show how a tandem of our \ac{WAPT} result and dynamical system theory yields
    the desirable convergence of various existing schemes, as well as motivates more efficient
    algorithms that enjoy the same rigorous guarantees.
\end{enumerate}

\para{Related work}

There is a vast literature on \emph{structured} non-convex
sampling, where one imposes extra assumptions on the target density. %
Under these conditions, one
can derive \emph{non-asymptotic} rates for Langevin-based schemes
\citep{raginsky2017non, cheng2018sharp, li2019stochastic, xu2018global, vempala2019rapid,
zou2019sampling, majka2020nonasymptotic, chewi2021analysis, ma2021there, mou2022improved}. 
Our work is orthogonal to these works as we study \emph{generic} non-convex sampling, an NP-hard problem whose convergence is asymptotic at best.

Most relevant to our paper are the works \cite{lamberton2002recursive, teh2016consistency,
benaim2017ergodicity, durmus2017nonasymptotic, balasubramanian2022towards}, which study the
asymptotic convergence of Langevin-based schemes under minimal regularity assumptions on $f$.
Compared to their results, our guarantees either improve upon existing ones or are incomparable; see
\cref{sec:comparison} for a more detailed comparison.

\section{The Langevin--Robbins--Monro Template}
\label{sec:template}

We consider the following general template for sampling algorithms: Starting from an
initial point, the iterates $\crl{\dpt_k}_{k\in\N}$ follow the recursion
\begin{equation}\label{eq:standard-template}\tag{LRM}
  \dpt_{k+1} = \dpt_k - \step_{k+1} \crl*{ v(\dpt_k) + \pert_{k+1}} +
  \sqrt{2\step_{k+1}}\,\sigma(x_k)\,\stdgauss_{k+1},
\end{equation}
where $\step_k$'s are step sizes, $v$ is a vector field, $\pert_k$'s are (random or deterministic)
perturbations, $\sigma$ is the state-dependent diffusion matrix, and $\stdgauss_k$'s are \iid
standard Gaussian random variables. In the sequel, we will further decompose the perturbation as
$\pert_k = \noise_k + \bias_k$, where $\noise_k$ is the (zero-mean) \emph{noise} and $\bias_k$ is
the \emph{bias}. We call this recursion the \acdef{LRM} template, as it is reminiscent of the
Robbins-Monro template for stochastic approximation
\citep{robbins1951stochastic}.

The generality of the \ac{LRM} template allows us to capture many existing algorithms and suggests
ways to design new ones. For illustration purposes, we showcase instances of
\eqref{eq:standard-template} with the following examples. Other examples (SGLD and proximal) are
provided in \cref{app:examples}. In the first three examples, the vector field $v$ in
\eqref{eq:standard-template} is $-\grad f$ and $\sigma \equiv 1$.

\begin{example}\label{ex:RMM}
  The \emph{Randomized Mid-Point Method}
  \citep{shen2019randomized,he2020ergodicity}
  is an alternative discretization scheme to Euler-Maruyama and has been proposed for both overdamped and
  underdamped Langevin diffusion. For the overdamped case, its iterates are
  \begin{equation}\label{eq:randomized-midpoint}\tag{RMM}
    \begin{aligned}
      \xmid{k} &= \dpt_k - \step_{k+1}\alpha_{k+1}\stochgrad f(x_k) +
        \sqrt{2\step_{k+1}\alpha_{k+1}}\, \stdgauss_{k+1}', \\
      x_{k+1} &= \dpt_k - \step_{k+1}\stochgrad f(\xmid{k}) +
        \sqrt{2\step_{k+1}}\, \stdgauss_{k+1},    
    \end{aligned}
  \end{equation}
  where $\crl{\alpha_k}$ are \iid and uniformly distributed in $[0,1]$, $\stdgauss_k,
  \stdgauss_k'$ are standard Gaussian random variables with cross-variance $\sqrt{\alpha_k}I$, and $\stochgrad f$ is a noisy
  evaluation of $\nabla f$. To cast \eqref{eq:randomized-midpoint} in
  the \ac{LRM} template, we set $\noise_{k+1} \coloneqq \stochgrad f(\xmid{k}) - \grad f(\xmid{k})$
  and $\bias_{k+1} \coloneqq \grad f(\xmid{k}) - \grad f(\dpt_k)$.$\hfill\qqed$
\end{example}

\begin{example}
\label{ex:ORMM}
  Inspecting the update rule of \eqref{eq:randomized-midpoint}, we see that it requires \emph{two}
  gradient oracle calls at each iteration. Inspired by the \emph{optimistic gradient methods} in
  optimization and online learning \citep{popov1980modification, chiang2012online,
  rakhlin2013optimization}, we propose to ``recycle'' the past gradients:
  \begin{equation}\label{eq:optimistic-randomized-midpoint}\tag{ORMM}
    \begin{aligned}
      \xmid{k} &= \dpt_k - \step_{k+1}\alpha_{k+1}\stochgrad f(x_{k-1/2}) +
      \sqrt{2\step_{k+1}\alpha_{k+1}}\stdgauss_{k+1}',\\
      x_{k+1} &= \dpt_k - \step_{k+1}\stochgrad f(\xmid{k}) +
      \sqrt{2\step_{k+1}}\, \stdgauss_{k+1},
    \end{aligned}
  \end{equation}
  where $\crl{\alpha_k}$, $\stdgauss_k, \stdgauss_k'$, and $\stochgrad f$ are the same as in
  \eqref{eq:randomized-midpoint}. This is again an \ac{LRM} scheme with 
  $\noise_{k+1} \coloneqq \stochgrad f(\xmid{k}) - \grad f(\xmid{k})$ and 
  $\bias_{k+1} \coloneqq \grad f(\xmid{k}) - \grad f(\dpt_k)$. 
  
  Notice that \eqref{eq:optimistic-randomized-midpoint} requires \emph{one} gradient oracle, thereby reducing the per-iteration cost of \eqref{eq:randomized-midpoint} by 2. To our knowledge, the scheme \eqref{eq:optimistic-randomized-midpoint} is new.$\hfill\qqed$
\end{example}

\begin{example}\label{ex:SRK-LD}
  In addition to the simple (stochastic) Euler-Maruyama discretization in \eqref{eq:sgld}, there exists a class of more sophisticated discretization methods of \eqref{eq:LD} known as higher-order integrators. The
  \emph{Stochastic Runge-Kutta method} \citep{liStochasticRungeKuttaAccelerates2020} is an example of an
  order $1.5$ integrator, with iterates 
  \begin{align*}
    h_1 &= x_k + \sqrt{2\step_{k+1}}\,\prn{c_1\xi_{k+1} + c_2\xi_{k+1}'} \\
    h_2 &= x_k - \step_{k+1}\stochgrad f(x_k) + \sqrt{2\step_{k+1}}\, \prn{c_3\xi_{k+1} + c_2\xi_{k+1}'}, \\
    x_{k+1} &= x_k - \tfrac{\step_{k+1}}{2}(\stochgrad f(h_1) + \stochgrad f(h_2)) +
    \sqrt{2\step_{k+1}}\,\xi_{k+1},
  \end{align*}
  where $\xi_{k+1}$ and $\xi_{k+1}'$ are independent standard Gaussian random variables, and
  $c_1,c_2,c_3$ are suitably chosen integrator constants. This algorithm is an \ac{LRM} scheme with
  $\noise_{k+1} \coloneqq \frac{1}{2}(\stochgrad f(h_1) - \grad f(h_1)) +  \frac{1}{2}( \stochgrad f(h_2))-
  \grad f(h_2))$ and 
  $\bias_{k+1} \coloneqq \frac{1}{2}(\grad f(h_1) + \grad f(h_2)) - \grad f(x_k)$.
\end{example}

\begin{example}\label{ex:Mirror-LD}
  The \emph{Mirror Langevin} algorithm \citep{hsieh2018mirrored,ahn2021efficient, zhang2020wasserstein}, which is the sampling analogue of the celebrated
  mirror descent scheme in optimization \cite{nemirovsky1983problem, beck2003mirror}, uses a strongly convex function $\phi$ to adapt to a favorable local geometry. In the dual space (\ie the image of $\grad \phi$), its iterates follow
  \begin{equation}\label{eq:mirror-ld}\tag{ML}
    x_{k+1} = x_k - \step_{k+1} \grad f(\grad \phi^*(x_k)) + \sqrt{2\step_{k+1}}\prn{\grad^2
    \phi^*(x_k)^{-1}}^{1/2}\,\stdgauss_{k+1},
  \end{equation}
  where $\phi^*$ is the \emph{Fenchel dual} of $\phi$ \cite{rockafellar2015convex}. In our framework, \eqref{eq:mirror-ld} fits
  into \eqref{eq:standard-template} by taking $v = -\grad f \circ \grad \phi^*$ and $\sigma =
  \prn{\grad^2 \phi^*}^{-1/2}$. %
  Additionally, one can also consider a stochastic version of \eqref{eq:mirror-ld} with noisy evaluations of
  $\grad f$.
\end{example}

\section{Technique Overview: A Dynamical System Perspective}
\label{sec:intuition}

The goal of our paper is to provide last-iterate guarantees for the general \ac{LRM} schemes
introduced in \cref{sec:template}. There are two equivalent, commonly considered, ways of
characterizing the dynamics of the iterates of an \ac{LRM} scheme. The first one is to view the
iterates $\crl{x_k}_{k\in\N}$ as a \emph{random} trajectory in $\R^d$, which is perhaps the most
natural way of describing a sampling algorithm. The second way is to view the \emph{distributions}
$\crl{\rho_k}_{k\in\N}$ of $\crl{x_k}_{k\in\N}$ as a \emph{deterministic} trajectory in the
\emph{Wasserstein space}. With these two characterizations in mind, in this section, we will devise
a new framework based on the dynamical system theory and present its high-level ideas.

To understand our novelty, it is important to contrast our framework to the existing Wasserstein
viewpoint towards Langevin-based sampling algorithms. Following the seminal work of
\citet{otto2001geometry}, one can view a sampling algorithm as the discretization of a class of
well-studied dynamical systems---\emph{gradient flows}. This viewpoint suggests using
\emph{Lyapunov} arguments, which has become the predominant approach in much prior work.

Despite its appealing nature, in the rest of this section, we will argue that Lyapunov analysis of
gradient flows is in fact \emph{not} suited for studying generic non-convex sampling. In particular,
we will show how our new framework is motivated to overcome the several important limitations of
gradient flow analysis. Finally, we give a high-level overview of the techniques used in our paper.

\para{Langevin Diffusion as Gradient Flows}

We denote by $\rho_t$ the probability density of $\Lgv_t$ in \eqref{eq:LD}, and consider
the continuous curve $t \mapsto \rho_t$ in the Wasserstein space $\Wsp$. In their seminal works,
\citet{jordan1998variational} and \citet{otto2001geometry} discover that this curve is the (exact)
gradient flow of the relative entropy functional; that is, defining the functional $F:\rho \mapsto
\KL{\rho}{e^{-f}}$, one has $\partial_t \rho_t = -\rgrad F(\rho_t)$,
where ``$\rgrad$'' is the gradient in the Wasserstein sense. This gradient flow viewpoint of
\eqref{eq:LD} thus provides a clear link between sampling in $\R^d$ and optimization in $\Wsp$.
Indeed, this suggests that the relative entropy is a natural choice for the \emph{Lyapunov function} of the
discrete-time sampling algorithm, which is a prominent approach for analyzing sampling
algorithms in recent years \citep{wibisono2018sampling, durmus2019analysis,
balasubramanian2022towards}.

Although the gradient flow viewpoint has led to a sequence of breakthroughs, it has a few important
shortcomings:
\begin{enumerate}[(a)]
  \item \label{iss:lyapunov-type-analysis}
    \looseness -1 The usual Lyapunov-type analysis for sampling algorithms focuses on bounding the
    change in relative entropy across iterations. This is extremely challenging when one considers
    more advanced sampling algorithms, as one has to understand the effect of the additive bias and
    noise of the algorithm  on the change of relative entropy. Crucially, this makes the Lyapunov
    analysis applicable only to the simple Euler-Maruyama discretization of
    \eqref{eq:LD},\footnote{
     While the Lyapunov-type analysis has been applied to elementary (\ie unbiased) discretization
     schemes for other SDEs, such as the under-damped (\ie kinetic) Langevin dynamics
     \cite{dalayan2020}, our primary focus in this paper remains centered on the over-damped
     Langevin diffusion and similar SDEs.
    }
    \ie
    $x_{k+1} = x_k - \step_{k+1} \grad f(x_k) + \sqrt{2\step_{k+1}} \, \xi_{k+1}$,
    and fails to capture more advanced and \emph{biased} sampling schemes such as
    \crefrange{ex:RMM}{ex:Mirror-LD}. Even for the simple \eqref{eq:sgld}, the presence of stochastic
    gradients significantly complicates the Lyapunov analysis and requires extra assumptions such as
    convexity \citep{durmus2019analysis} or uniform spectral gap \citep{raginsky2017non}.

  \item \label{iss:average-iterate}
    This gradient flow-based analysis often requires an extra \emph{averaging} step to decrease the
    relative entropy (see, \eg \citep{balasubramanian2022towards}). This is the main reason why
    many existing works provide guarantees only on the \emph{averaged} iterates ($\bar{\rho}_{k}
    \coloneqq \frac{1}{k}\sum_{i=1}^k \rho_i$) instead of the last ones ($\rho_k$).
\end{enumerate}

In this paper, we overcome these limitations by introducing a new perspective, whose two ingredients
are as follows.

\para{Wasserstein Asymptotic Pseudotrajectories}
A notion that will play a pivotal role in our analysis is the \acdef{WAPT}, which is a measure of
``asymptotic closeness'' in the Wasserstein sense, originally defined by
\citet{benaim1996asymptotic} for metric spaces:
\begin{definition}[\acl{WAPT}]\label[definition]{def:wapt}
  We say the stochastic process $\prn{X_t}_{t\geq 0}$ is a \acdef{WAPT} of the SDE 
  \begin{equation}\label{eq:sde}\tag{SDE}
    \dd \Phi_t = v(\Phi_t)\dd t + \sigma(\Phi_t)\dd B_t
    \textstyle
  \end{equation}
  if, for all $T>0$,
  \begin{equation}\label{eq:wapt-def}
    \lim_{t\to \infty} \sup_{0 \leq s \leq T} \Wtwo(X_{t + s}, \Xfl{s}{t}) = 0.
  \end{equation}
  Here, $\Xfl{s}{t}$ is the solution of the SDE at time $s$ initialized at $X_t$, and $\Wtwo$ is the
  2-Wasserstein distance.
\end{definition}

Despite the seemingly convoluted definition, \ac{WAPT} can be intuitively understood as follows: Let
$\{x_k\}_{k\in\N}$ be the iterates of a sampling scheme. Then, \eqref{eq:wapt-def} simply posits
that for sufficiently large $m$, one cannot distinguish between the ``tail'' iterates
$\{x_k\}_{k\geq m}$ versus the SDE solution \emph{starting at $x_m$}, up to arbitrarily small
error measured in terms of the Wasserstein distance. Since we are only interested in the asymptotic
behavior of $x_k$, these controls on the tail iterates will suffice to conclude the last-iterate
convergence.%
\footnote{\cref{def:wapt} is phrased in terms of a continuous-time stochastic process
  $(X_t)_{t\geq 0}$. The discrete iterates $\crl{x_k}_{k \in \N}$ can be converted to a
continuous-time process through a suitable interpolation; see \eqref{eq:interpolation}.}

Importantly, from the perspective of \ac{WAPT}, the Langevin diffusion \eqref{eq:LD} (or more
generally, $\Xfl{s}{t}$) is simply viewed as a generic dynamical system and \emph{not} as a gradient
flow. In particular, \emph{relative entropy will play no role throughout our analysis}, thereby
resolving issue \ref{iss:average-iterate}.%

\para{\acl{LRM} Schemes}
We have seen that the \ac{LRM} template in \cref{sec:template} is capable of capturing a broad range
of existing and new algorithms in a unified way. To resolve the remaining issue
\ref{iss:lyapunov-type-analysis}, we will further rely on the \ac{LRM} template: for proving that
\eqref{eq:standard-template} generates a \ac{WAPT} of the corresponding SDE, we show that the key
condition \eqref{eq:wapt-def} in \ac{WAPT} can be reduced to checking an easy-to-verify bound on the
perturbation terms $\pert_k$.

To achieve this, the most important step in our proof, which distinguishes our analysis from all
existing works in non-convex sampling, is the construction of the so-called \emph{Picard process},
the natural generalization of the Picard's successive approximation method
\citep{coddington1955theory} from ordinary differential equations to \emph{stochastic} differential
equations.
In the stochastic approximation literature, similar techniques have been successfully applied to
study optimization and games in various settings such as on Riemannian or primal-dual spaces
\citep{hsieh2021limits, karimi2022dynamics, mertikopoulos2022learning}. The application to sampling has also been previously explored by \citet{chau2021stochastic, bubeck2018sampling} in different contexts. What distinguishes our work from the existing literature is the advantage of generalizing the Picard process to encompass a vastly wider class of algorithms, specifically the \ac{LRM} schemes. Moreover, the integration of the Picard process with the theory of \ac{WAPT} plays a pivotal role in our analysis, and both of these aspects present original contributions.

\begin{figure}[t]
  \centering
  \begin{tikzpicture}[>=stealth,every node/.style={shape=rectangle,draw,rounded corners,font=\scriptsize},]
  \node at (0,-1.3) [draw=blue!30] (apt) {\begin{tabular}{c} Dynamics \\ via \ac{WAPT} \\ (\cref{thm:APT}) \end{tabular}};
  \node at (6.3,-1.0) [draw=red!30] (diss) {\begin{tabular}{c} Dissipativity \\
  (\cref{thm:LRMstability}) \end{tabular}};
  \node at (6.3,-2.2) [draw=red!30] (wdiss) {\begin{tabular}{c} Weak \\ Dissipativity \\
  (\cref{thm:LRMstability2}) \end{tabular}};
  \node at (3.5,-1.3) [draw=blue!30] (stab) {\begin{tabular}{c} Stability \\ (\cref{eq:stability}) \end{tabular}};

  \begin{scope}[on background layer]
    \node [fit=(apt) (stab),dashed,dash pattern=on 1mm off 1mm] (scp)
    [label=180:{\begin{tabular}{l}\textit{Dynamical}\\ \textit{Perspective} \end{tabular}}] {};
  \end{scope}

  \node at (1.85, -1.5) (plus) {+};
  \node at (1.85, -2.9) [draw=black!30!green!100] (last) {\begin{tabular}{c} Last-Iterate Convergence \\ in
  $\Wsp$ \\(\cref{thm:APTtoICT}) \end{tabular}};

  \draw [->,double] (diss.west) to  (stab.10);
  \draw [->,double] (wdiss.west) to (stab.350);
  
  \draw [-] (apt.east) to [out=0,in=180] (plus.west);
  \draw [-] (stab.west) to [out=180,in=0] (plus.east);
  
  \draw [->,double] (plus.south) to (last.north);
\end{tikzpicture}
  \caption{High-level overview of the two components of the dynamical perspective.}
  \label{fig:high-level-diagram}
\end{figure}
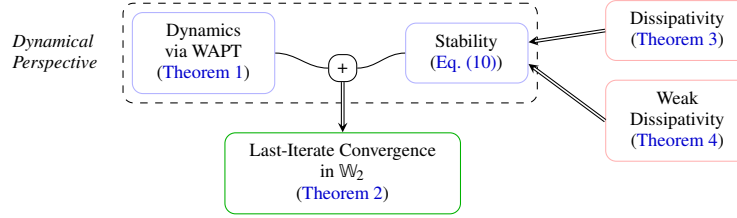

\para{Framework overview}

To conclude, for proving last-iterate convergence, we proceed as follows:

\begin{enumerate}
  \item For a given \ac{LRM} scheme $\crl{\dpt_k}_{k\in\N}$, we first construct a continuous-time
    trajectory $(\cpt_t)_{t\geq 0}$ via \emph{interpolating} the iterates (see \eqref{eq:interpolation}). 
  \item We prove that $(\cpt_t)$ constitutes a \ac{WAPT} of the SDE (see \cref{thm:APT}).
    This step relies heavily on the construction of the aforementioned Picard process.
  \item By invoking the dynamical system theory of \citet{benaim1996asymptotic}, 
    the convergence of \ac{LRM} schemes reduces to simply checking the \emph{stability condition}
    (\cref{thm:APTtoICT}). In the Wasserstein space, this condition translates into boundedness of
    the second moments of the iterates $\crl{\dpt_k}$, for which there is a plethora of approaches;
    we present two such methods in \cref{sec:stability}. 
\end{enumerate}

\cref{fig:high-level-diagram} depicts a high-level overview of the ingredients needed in our
framework, and their corresponding theorems.

\section{The Dynamics of Langevin--Robbins--Monro Schemes}
\label{sec:theory-statement}

\looseness -1 In this section, we view \eqref{eq:standard-template} as a noisy and biased
discretization of \eqref{eq:LD}. To make this analogy precise, let $(B_t)_{t\geq 0}$ be a Brownian
motion defined on a filtered probability space with filtration $(\Fil_t)_{t\geq 0}$, and define
$\tau_k = \sum_{n=1}^{k} \step_n$ to be the effective time that has elapsed at the iteration $k$. 
Using the Brownian motion, we can rewrite \eqref{eq:standard-template} as 
\begin{equation}\label{eq:standard-template-with-bm}
  \dpt_{k+1} = \dpt_k - \step_{k+1} \crl*{ v(\dpt_k) + \pert_{k+1}} +
  \sigma(\dpt_k) \,(B_{\tau_{k+1}} - B_{\tau_k}),
\end{equation}
assuming that the filtration satisfies $Z_k \in \Fil_{\tau_k}$.%
\footnote{One can augment the filtration of the Brownian motion by the $\sigma$-algebra of $Z_k$ at
times $\crl{\tau_k}_{k\in\N}$.}
The (continuous-time) \emph{interpolation} $(\cpt_t)_{t\geq 0}$ of $\crl{x_k}_{k\in\N}$ is then
defined as the adapted process
\begin{equation}\label{eq:interpolation}
  \begin{aligned}
    \cpt_t &= \dpt_k - \prn*{t- \tau_k}\crl{v(\dpt_k) + \exof{\pert_{k+1}
    \given \Fil_{t}}} + \sigma(\dpt_k)\,(B_t - B_{\tau_k}),\quad \text{for}\  t\in[\tau_k, \tau_{k+1}].
  \end{aligned}
\end{equation}
In addition, for a fixed $t$, consider the Brownian motion
$(B_s^{(t)})_{s\geq 0}$ where $B_s^{(t)} \coloneqq B_{t+s} - B_t$, and define the
\emph{\meandynamics} $(\Xfl{s}{t})_{s\geq 0}$ as the (strong) solution of \eqref{eq:sde} initialized
at $X_t$.
It is important to note that $\Xflt{t}$ and
$\cpt$ are synchronously coupled by sharing the same Brownian motion.

\subsection{Technical Assumptions and Requirements}

We now introduce the basic technical assumptions and discuss their generality.

\begin{assumption}\label{ass:f}
  The vector field $v$ is $L$-Lipschitz, and satisfies 
  $\inner{x, v(x)} \leq C_v(1 + \nrm{x})$ for some $C_v > 0$. Moreover, $\sigma$ is $L$-Lipschitz
  and is bounded in Hilbert-Schmidt norm.
\end{assumption}
Lipschitzness of $v$ is a standard assumption and is also required to ensure the existence of a
unique strong solution of \eqref{eq:sde}. The second assumption on the vector field is exceedingly
weak and when $v = -\grad f$, is satisfied even for distributions without moments. The assumptions
on diffusion coefficient $\sigma$ are already satisfied when $\sigma \equiv 1$, and we show
that it holds for practical schemes such as \cref{ex:Mirror-LD}.

\begin{assumption}\label{ass:step-size}
  The Robbins-Monro summability conditions hold: $\sum_{k=1}^\infty \step_k = \infty$ and
  $\sum_{k=1}^\infty \step_k^2 < \infty$. Moreover, for some constant $P$ to be defined in
  \eqref{eq:definition-p}, we have
  \begin{equation}\label{eq:strange-step-size}
    \step_{k+1}/\step_{k} + P\step_{k}\step_{k+1} < 1 - \step_{k}, \quad \forall k.
  \end{equation}
\end{assumption}

The Robbins-Monro step size conditions are standard in the non-convex sampling literature
\citep{lamberton2002recursive,lemaire2005estimation,durmus2017nonasymptotic,balasubramanian2022towards}.
For  \eqref{eq:strange-step-size}, it can be verified that condition is satisfied even for
slowly-decreasing step sizes such as $\step_k \propto (\sqrt{k}\log k)^{-1}$, which hence is not
restrictive.

\begin{assumption}\label{ass:noise-bias}
  The noises $\crl{\noise_k}_{k\in\N}$ form a martingale difference sequence, \ie $\exof{\noise_{k+1}
  \given \noise_k} = 0$, and have uniformly bounded second moments.  In addition, the bias terms satisfy
  \begin{equation}\label{eq:bias-assumption}
  \exof{\nrm{\bias_{k+1}}^2 \given \Fil_{\tau_k}} 
  = \bigoh(\step_{k+1}^2\nrm{v(x_k)}^2 + \step_{k+1}).
  \end{equation}
\end{assumption}
A martingale difference sequence is more general than an \iid sequence, allowing the noise to be
state-dependent. The bias condition \eqref{eq:bias-assumption} simply states that the bias shall not
overpower the signal $v(x_k)$, and, as we show later, is satisfied by all our examples.

\subsection{From Discrete to Continuous: \texorpdfstring{\ac{LRM}}{LRM} Schemes and \texorpdfstring{\acp{WAPT}}{WAPTs}}

We are now in a position to state our main theorems. Our first result below establishes a precise
link between the discrete-time \eqref{eq:standard-template} and the continuous-time \eqref{eq:sde}.

\begin{theorem}\label{thm:APT}
  Under \cref{ass:f,ass:noise-bias,ass:step-size}, the interpolation \eqref{eq:interpolation} of an
  \ac{LRM} scheme %
  is a \acl{WAPT} of \eqref{eq:sde}.
\end{theorem}

\para{Sketch of the Proof for \cref{thm:APT}}

The proof of this theorem is heavily based on the notion of the Picard process and iterate moment
bounds. The complete proof can be found in \cref{app:proofs-lemmas}.

\para{Step 1: The Picard Process}
For a fixed $t > 0$, recall the construction of the interpolation \eqref{eq:interpolation} and the
\meandynamics. Central to our analysis is the \emph{Picard process}, defined as
\begin{equation}\label{eq:picard-sde}
  \Xpi{s}{t} = X_t + \int_0^s v(\cpt_{t+u})\dd u + \int_0^s \sigma(\cpt_{t+u})\, \dd B_u^{(t)}.
\end{equation}
The Picard process is adapted and is (synchronously) coupled with
the \meandynamics and the interpolation. We think of the Picard process as one step of the
\emph{Picard iteration} for successive approximations to solve ODEs. This means, intuitively,
that its trajectory should be close to the original interpolation, as well as to that of the
\meandynamics, playing the role of a ``bridge''.

Fix $T > 0$. For $s \in [0, T]$, we decompose the distance between the interpolation $X_t$ in
\eqref{eq:interpolation} and the \meandynamics~as
\begin{equation}\label{eq:main-decomp}
  \tfrac{1}{2}\nrm{\cpt_{t + s} - \Xfl{s}{t}}^2 
  \leq \nrm{\Xpi{s}{t} - \Xfl{s}{t}}^2 + \nrm{\cpt_{t+s} - \Xpi{s}{t}}^2.
\end{equation}
We now bound each term of the decomposition. By synchronous coupling of the processes, 
Lipschitzness of $v$, and It\^{o} isometry, \cref{lem:bound-dist-picard-mean} bounds the first term as
\begin{equation}\label{eq:bound-gronwall}
  \nrm{\Xpi{s}{t} - \Xfl{s}{t}}^2 \leq 2(T+1) L^2 \int_0^s \nrm{\Xfl{u}{t} - \cpt_{t+u}}^2 \,\dd u.
\end{equation}
This will be suitable for later use of Gr\"onwall's lemma.

\para{Step 2: Accumulated Noise and Bias}
For the rest of the proof, we need some extra notation. Define $m(t) \coloneqq \sup\{ k \geq 0 :
\tau_k \leq t\}$ and the piecewise-constant process $\Xbar_t \coloneqq \dpt_{m(t)}$. Going back to
the second term of \eqref{eq:main-decomp}, observe that 
\begin{align}
  \cpt_{t+s} - \Xpi{s}{t} 
  = \int_t^{t+s} v(\Xbar_u) - v(X_{u}) \dd u + \int_0^{s}\sigma(\Xbar_{t+u}) - \sigma(X_{t+u})\dd
  B_u^{(t)} - \Delta_\pert(t,s), \label{eq:picard-minus-interpolation}
\end{align}
where $\Delta_\pert(t,s)$ is the accumulated noise and bias from time $t$ to time $t+s$.
It is expected that $\nrm{\Delta_\pert(t,s)}$ eventually becomes negligible, since the step size
becomes small. The next lemma confirms this intuition.

\begin{restatable}{lemma}{lembounddeltaZ}\label{lem:sup-bound-bias-noise-cumulative}
  Suppose \cref{ass:f,ass:noise-bias,ass:step-size} hold. Then, for any fixed $T > 0$ we have
  \[ 
    \lim_{t\to \infty} \sup_{0\leq s \leq T} \E\nrm{\Delta_Z(t, s)}^2 = 0.
  \]
\end{restatable}
\para{Step 3: Gradient Moment Bounds}
Based on \eqref{eq:picard-minus-interpolation} and \cref{lem:sup-bound-bias-noise-cumulative},
bounding the distance between the Picard process and the interpolation essentially reduces to
bounding how much the discrete algorithm ``moves'' during one iteration in expectation. This, in
turn, depends on how large the moments of $\nrm{v(x_k)}$ grow per iteration, which is controlled by
the following lemma:
\begin{restatable}{lemma}{gradientboundltwo}
  \label{lem:gradient-bound-l2}
  Let $\crl{\dpt_k}_{k\in\N}$ be the iterates of \eqref{eq:standard-template} and suppose
  \cref{ass:f,ass:noise-bias,ass:step-size} hold. Then, $\E\nrm*{\dpt_k}^2 =
  \bigoh(1/\step_{k+1})$.
  This in turn implies $\E\nrm*{v(\dpt_k)}^2 = \bigoh(1/\step_{k+1})$ and
  $\E\nrm{\bias_{k+1}}^2 = \bigoh(\step_{k+1})$.
\end{restatable}

Using this lemma and \cref{lem:sup-bound-bias-noise-cumulative}
we can obtain $A_t := \sup_{0 \leq s \leq T} \E\nrm{\cpt_{t+s} - \Xpi{s}{t}}^2
\to 0$ as $t\to\infty$, which shows that the Picard process gets arbitrarily close to the
interpolation as $t\to\infty$.

\para{Step 4: Concluding the Proof}
Let us go back to the decomposition \eqref{eq:main-decomp}. Taking expectation and using
\eqref{eq:bound-gronwall} and Gr\"onwall's lemma, we obtain
$\E\brk{\nrm{\cpt_{t + s} - \Xfl{s}{t}}^2 } \leq 4\,A_t \exp(T^2L^2),$ Thus,
\[
  \lim_{t \to \infty} \sup_{s\in [0, T]} \E\brk*{ \nrm{X_{t + s} - \Xfl{s}{t}}^2 } = 0.
\]
As we coupled $\cpt_{t+s}$ and $\Xfl{s}{t}$ in a specific way (via synchronizing the Brownian
motions), we directly get an upper bound on the Wasserstein distance. \hfill\qed

\section{Last-Iterate Convergence of Sampling Schemes}
\label{sec:stability}
\acused{LRM}
\acused{APT}

In this section we focus on last-iterate convergence of \ac{LRM} schemes in Wasserstein space. We
first explore the interplay between the convergence of \acp{WAPT} and \emph{stability}. We then show
that the existing stability results for simple Euler-Maruyama discretization of the Langevin
diffusion can be extended, with little to no extra assumptions, to the class of \ac{LRM} schemes in
\cref{sec:template}. This in turn readily implies the last-iterate convergence of a wide class of
\ac{LRM} schemes.

\subsection{From \texorpdfstring{\acp{WAPT}}{WAPTs} to Convergence in Wasserstein Distance}
Since convergence of the distribution of $\dpt_k$ to $\pi$ in Wasserstein distance implies
convergence of the second moments of $\dpt_k$ to that of $\pi$
\citep{ambrosio2005gradient}, convergence in the Wasserstein space should at least require:
\begin{align}
\label{eq:stability}
\sup_{k \in \N}\ \E\nrm{x_k}^2 < \infty.
\end{align}
It turns out that, for \acp{WAPT}, the exceedingly weak necessary condition \eqref{eq:stability} is
also \emph{sufficient}:
\begin{theorem}\label{thm:APTtoICT}
  Let $(X_t)$ be a \acl{WAPT} of the Langevin diffusion \eqref{eq:LD} generated by an \ac{LRM} scheme
  $\{x_k\}$ via \eqref{eq:interpolation}. Then $W_{2-\eps}(x_k, \pi) \to 0$ for all $\eps \in (0,1]$ if and only if \eqref{eq:stability} holds.\footnote{In the accepted version of this article, it is mistakenly stated that $W_2(x_k, \pi) \to 0$. The proof is adapted to the correct version of the theorem.}
\end{theorem}
\begin{proof}

The proof relies on the structure of compact sets in the Wasserstein space and limit-set theorems
for dynamical systems \cite{benaim1996asymptotic}.
Specifically,

Specifically, the bounded second moments imply that the sequence $\crl{\law{x_k}}_k$ is tight, and \citep[Eq.~(5.1.20)]{ambrosio2005gradient} implies that it is $(2-\eps)$-uniformly integrable. Thus, by \citep[Prop.~7.1.5]{ambrosio2005gradient}, the mentined sequence is pre-compact in the $(2-\eps)$-Wasserstein topology.
Moreover, \cref{ass:f} implies that the
\meandynamics is globally integrable. Thus, $(\law{X_t})_{t\geq 0}$ is a pre-compact \ac{WAPT} of a
globally integrable flow, and we can apply the limit-set theorem for metric spaces \citep[Theorem
0.1]{benaim1996asymptotic} to conclude that the limit-set of $(\mathrm{law}(X_t))_t$ is an
\emph{internally chain transitive (ICT) set} in $\mathbb{W}_{2-\eps}$.

Next, we show that for the case of the \meandynamics, the only ICT set is $\{\pi\}$,
implying the desired convergence of our theorem. To see this,
define $V(\cdot) = D_{\mathrm{KL}}(\cdot \mid \pi)$. It can be observed that $V$ is a Lyapunov function
for \eqref{eq:LD}, whose value is strictly decreasing along the flow (as the time derivative
of $V$ along the flow is negative of the relative Fisher information, which is strictly positive for
all measures other than $\pi$). Thus, all requirements of \citep[Prop. 6.4]{benaim1999dynamics} are
satisfied, showing that the only point in the ICT set is $\pi$. This also shows the uniqueness of
the stationary distribution of $\eqref{eq:LD}$.
\end{proof}

\begin{remark*}
  From the proof of \cref{thm:APT}, we observe that the supremum of the Wasserstein distance between
  $(X_\cdot)_{[t, t+T]}$ and $(\Xfl{\cdot}{t})_{[0, T]}$ typically scales exponentially with $T$,
  which is common for weak approximation error in the literature, see \citep{milstein2004stochastic}.
  Despite the exponential dependence on $T$, the convergence of the last iterate is assured by
  \cref{thm:APTtoICT} without a need of a uniform control in $T$. This is primarily attributed to the
  adoption of a dynamical system viewpoint and the application of corresponding tools, effectively
  harnessing the paradigm established by Bena\"im and Hirsch.

\end{remark*}

\crefrange{thm:APT}{thm:APTtoICT} in tandem thus show that, as long as an \ac{LRM} scheme satisfies
\crefrange{ass:f}{ass:noise-bias} and the moment condition \eqref{eq:stability}, the desirable
last-iterate convergence in $\Wsp$ is immediately attained. Therefore, in the rest of this section,
we turn our focus to establishing \eqref{eq:stability} for \ac{LRM} schemes.

\subsection{Bounded Moments of \texorpdfstring{\ac{LRM}}{LRM} Schemes}

There is a long history of study on conditions that ensure \eqref{eq:stability} for iterative
algorithms, which has culminated in the so-called \emph{dissipativity} properties. We consider two
such examples below.

\begin{assumption}[Dissipativity]
\label{asm:dissipativity}
There exist constants $\alpha >0$ and $\beta\geq 0$ such that
  \[
    \inner{x, v(x)} \leq -\alpha \nrm{x}^{2} + \beta,\quad \forall x\in\R^d.
  \]
\end{assumption}
Under \cref{asm:dissipativity}, it is classical that \eqref{eq:stability} holds for the simple
Euler-Maruyama discretization of \eqref{eq:LD} with deterministic or stochastic gradient oracles
\citep{hale1988asymptotic, meyn1993stability, roberts1996exponential, lamberton2002recursive,
lemaire2005estimation, teh2016consistency, raginsky2017non}. These studies, however, cannot handle
\emph{non-zero bias}, which, as seen in \crefrange{ex:RMM}{ex:SRK-LD}, is crucial for incorporating
more advanced sampling schemes.

To this end, our next result shows that for a wide class of \ac{LRM} schemes,
the stability \eqref{eq:stability} essentially comes for free under \cref{asm:dissipativity}. The
proof is provided in \cref{app:stability}.
\begin{theorem}
\label{thm:LRMstability}
  Let $v$ be a vector field satisfying \cref{asm:dissipativity,ass:f} and $\sigma$ be a diffusion
  coefficient satisfying \cref{ass:f}, and let $\{x_k\}$ be an \ac{LRM} scheme. Assume that
  $\lim_{k\to\infty} \step_k = 0$, $\sup_k \E\nrm{\noise_k}^2 < \infty$, and the bias satisfies
  \eqref{eq:bias-assumption}. Then, the stability condition \eqref{eq:stability} holds for
  $\{x_k\}$.
\end{theorem}

A weaker notion of dissipativity that has been studied in the literature is:
\begin{assumption}[Weak dissipativity]
\label{asm:weakdissipativity}
There exist constants $\alpha >0$, $\kappa \in (0,1]$, and $\beta\geq 0$ such that
  \[
    \inner{x, v(x)} \leq -\alpha \nrm{x}^{1+\kappa} + \beta,\quad \forall x\in\R^d.
  \]
\end{assumption}
When $\kappa=1$, \cref{asm:weakdissipativity} is simply \cref{asm:dissipativity}. As opposed to
\cref{asm:dissipativity}, which requires \emph{quadratic growth} of $f$ outside a
compact set (when $v = -\grad f$), \cref{asm:weakdissipativity} only entails \emph{superlinear
growth} and therefore is considerably weaker.

\newcommand{\Ekc}[1][k+1]{ \exof{\nrm{x_{#1}}^2 \given \mathcal{F}_k } }
\newcommand{\Ek}[1][k]{ \E\nrm{x_{#1}}^2 }

For Euler-Maruyama discretization of \eqref{eq:LD} with deterministic gradients,
\cite{durmus2017nonasymptotic} prove that \cref{asm:weakdissipativity} is sufficient to guarantee
bounded moments of the iterates. As for a generic \ac{LRM} scheme, we consider the following general
condition on the bias terms, which will suffice to cover all our examples in \cref{sec:template}:
For some constant $c$,
\begin{equation}
\label{eq:bias-condition}
  \begin{aligned}
    \nrm{\bias_{k+1}}^2 \leq c \big(\step^2_{k+1} \nrm{v(x_k)}^2 +
  \step^2_{k+1}\nrm{\noise'_{k+1}}^2
    +\,\step_{k+1} \nrm{\xi_{k+1}'}^2 + \step_{k+1} \nrm{\xi_{k+1}}^2 \big),
  \end{aligned}
\end{equation}
where $\noise_{k+1}'$ is an extra noise term, and $\xi'_{k+1}$ is a standard Gaussian independent of
the noises and $\xi_{k}$.
The price to pay with the weaker \cref{asm:weakdissipativity}, however, is that we need to assume
sub-Gaussianity of the noise. For a proof, see \cref{app:stability}.

\begin{theorem}
\label{thm:LRMstability2}
  Let $\pi \propto e^{-f}$ be the target distribution, where $v = -\grad f$ satisfies
  \cref{asm:weakdissipativity,ass:f}, and let $\{x_k\}$ be an \ac{LRM} scheme. Assume
  that $\lim_{n\to\infty} \step_k = 0$, the noises $\noise_k$ and $\noise'_k$ are sub-Gaussian, and
  the bias term of $\{x_k\}$ satisfies \eqref{eq:bias-condition}. Then, \eqref{eq:stability} holds
  for $\{x_k\}$ in when
  \begin{enumerate*}[label={\emph{(\roman*)}}]
    \item $\sigma\equiv 1$, or
    \item $f$ is Lipschitz and the \ac{LRM} follows the Mirror Langevin algorithm
      (\cref{ex:Mirror-LD}).
  \end{enumerate*}
\end{theorem}

\subsection{Examples of Convergent \texorpdfstring{\ac{LRM}}{LRM} Schemes}
\label{sec:application}

We now illustrate the use of \cref{thm:APT,thm:APTtoICT,thm:LRMstability,thm:LRMstability2} on our examples in
\cref{sec:template}.

\begin{proposition}\label{prop:example-proofs}
  Under \cref{ass:f} and noise with uniformly bounded second moments, the following holds for
  \crefrange{ex:RMM}{ex:PLA}:
  \begin{enumerate*}[(i)]
    \item The bias has the form \eqref{eq:bias-condition} and satisfies \eqref{eq:bias-assumption},
    \item As a result, under \cref{ass:noise-bias,ass:step-size},  \crefrange{ex:RMM}{ex:PLA}
      produce iterates that generate a \ac{WAPT} of \eqref{eq:sde}.
    \item Under the additional conditions of \cref{thm:LRMstability} or \cref{thm:LRMstability2},
      \crefrange{ex:RMM}{ex:PLA} enjoy last-iterate convergence to the target distribution in
      Wasserstein distance.
  \end{enumerate*}
\end{proposition}

\subsection{Comparison to Existing Work}
\label{sec:comparison}

\begin{table*}[t]
\centering
{\scshape
\begin{footnotesize}
\begin{tabular}{ccccc}
\toprule[1.2pt]
  \makecell{} & Noise & Bias  & Last-Iterate \\
\midrule
  \makecell{\citet{lamberton2002recursive,lemaire2005estimation}}   & \xmark & \xmark &
  \xmark \\
\midrule
  \makecell{\citet{teh2016consistency}}   & \cmark & \xmark   & \xmark \\
\midrule
  \makecell{\citet{benaim2017ergodicity}}   & \xmark & \xmark   & \cmark  \\
\midrule
\makecell{\citet{durmus2017nonasymptotic}}   & \xmark & \xmark   & \cmark  \\
\midrule
\makecell{\citet{balasubramanian2022towards}}   & \xmark & \xmark  & \xmark   \\
\midrule
  \makecell{This work}   & \cmark & \cmark  & \cmark  \\
\bottomrule[1.2pt]
\end{tabular}
\end{footnotesize}
}
\caption{%
  Comparison to existing works on convergence of \ac{LRM} schemes. %
  All methods, except for \citep{balasubramanian2022towards}, require bounded second moments of the iterates.
  }
\label[table]{tab:comparison}
\vskip -0.1in
\end{table*}

We now give a more detailed comparison of our results to existing literature; a summary is given in
\cref{tab:comparison}, and additional comparison with prior works can be found in \cref{app:comparison}.

\para{Guarantees for \ac{LRM} Schemes}
\citet{lamberton2002recursive} and \citet{lemaire2005estimation} study the simple Euler-Maruyama
discretization of \eqref{eq:LD} with deterministic gradients (\ie $\noise_k = \bias_k=0$) and
establish the weak convergence of the average iterates under a moment condition that is slightly
weaker than \eqref{eq:stability}.%
\footnote{Although the condition in \citep{lamberton2002recursive,lemaire2005estimation} is stated
in a weaker form than \eqref{eq:stability}, it is typically only verified on a special case that is
equivalent to our \cref{asm:dissipativity}, and thus implies \eqref{eq:stability}. See \eg
\citep[Remark 3]{lamberton2002recursive}.}
Their analysis is further extended by \cite{teh2016consistency} to incorporate
stochastic gradients. Later, the last-iterate convergence of the simple Euler-Maruyama
discretization of \eqref{eq:LD} is studied by \cite{durmus2017nonasymptotic}, who prove the
convergence in the total variation distance under \cref{asm:weakdissipativity}. Another work on a
similar setting as \citep{durmus2017nonasymptotic} is \citep{benaim2017ergodicity}, where the
convergence criterion is given in an integral probability metric (IPM) \citep{muller1997integral} of
the form
$d_\mathcal{B}(\mu, \nu) \coloneqq \sup_{\varphi\in\mathcal{B}} \abs{\E_\mu \varphi -
\E_{\nu} \varphi}$
for a certain class of test functions $\mathcal{B}$ that is known to imply weak convergence, but not
convergence in total variation or Wasserstein distances.

Compared to these results, our guarantees possess the following desirable features:
\begin{itemize}%
\item The convergence is always on the last iterates instead of the average iterates.
\item As we tolerate biased algorithms, the class of \ac{LRM} schemes we consider is significantly more general
      than the ones in existing work.
\end{itemize}

Finally, we note that our results are incomparable to the recent work of
\citet{balasubramanian2022towards}, who derive the same result as in \citep{lamberton2002recursive,
lemaire2005estimation}, \ie average-iterate, weak convergence, deterministic Euler-Maruyama discretization.
A remarkable feature of the analysis in \citep{balasubramanian2022towards} is that it does not
require any bounded moments, and, in particular, their bounds can be applied to target distributions
with unbounded variance. However, the downside of \citep{balasubramanian2022towards} is that, in the
presence of $\noise_k$ and $\bias_k$, their analysis produces a bound that does \emph{not} vanish as
$k\to\infty$; see \citep[Theorem 15]{balasubramanian2022towards}. In contrast, our framework can
tolerate quite general $\noise_k$ and $\bias_k$, gives stronger guarantees ($\Wtwo$ \vs weak
convergence; last-iterate \vs average-iterate).

\para{On Analysis Techniques}
While, to our knowledge, our framework is significantly different from previous works on sampling, we
acknowledge that similar ideas of creating an auxiliary process in-between the iterates and the
continuous-time flow is not entirely new and has been touched upon in the literature, \eg
\cite{chau2021stochastic,bubeck2015sampling}. That being said, our specific approach in building the
Picard process and its development into a wider array of algorithms, \ie \acl{LRM} schemes,
undoubtedly plays a pivotal role in our analysis. Moreover, the integration of the Picard process
with the theory of asymptotic pseudo-trajectories offers dual benefits to our study, and we view
these as our unique contributions to this area of research.

Furthermore, the novel Picard process gives a significant advantage in all of our results. The work
of \cite{benaim2017ergodicity} also hinges on dynamical system theory-related ideas. Yet, missing
the critical step of the Picard process has seemingly resulted in much weaker findings compared to
our work. This observation is not meant as a critique; rather, it merely highlights the potency of
the unique method we have integrated into our study.

\section{Concluding Remarks}
\label{sec:future}
In this paper, we provided a new, unified framework for analyzing a wide range of sampling schemes,
thus laying the theoretical ground for using them in practice, as well as motivating new and more
efficient sampling algorithms that enjoy rigorous guarantees. We built on the ideas from dynamical
system theory, and gave a rather complete picture of the asymptotic behavior of many first-order
sampling algorithms. In short, our results help with the following:
\begin{itemize}
\item \textbf{Validating existing methods:} Methods like mirror Langevin and randomized mid-point
currently lack even asymptotic guarantees in fully non-convex scenarios, such as sampling from
neural network-defined distributions. Our work fills this gap by offering the first rigorous
justification for these schemes, supporting practitioners in utilizing these methods confidently.

\item \textbf{Facilitating new algorithm design:} Our work motivates novel sampling methods through
  a straightforward verification of \crefrange{ass:f}{ass:noise-bias}. An illustrative instance
involves the randomized mid-point method and Runge-Kutta integrators, wherein a substantial 50\%
reduction in computation per iteration can be achieved without compromising convergence by simply
recycling past gradients, shown in \cref{ex:ORMM}.
The balance between the benefits of saving gradient oracles and potential drawbacks remains an open
question, necessitating case-by-case practical evaluation. Nevertheless, our theory provides a
flexible algorithmic design template that extends beyond the current literature's scope.
\end{itemize}

While our \ac{WAPT} result holds under very mild conditions, a severe limitation of our current
framework is that it \emph{only} applies to Langevin-based algorithms, whereas there exist numerous
practical sampling schemes, such as Metropolis-Hastings, that are not immediately linked to
\eqref{eq:LD}. We believe that this restriction arises as an artifact of our analysis, as the
\ac{WAPT} framework can in principle be applied equally well to any continuous-time dynamics.
Lifting such constraint is an interesting future work.

\begin{ack}
This work was supported by the European Research Council (ERC) under the European Union's Horizon
2020 research and innovation program grant agreement No 815943. YPH acknowledges funding through an
ETH Foundations of Data Science (ETH-FDS) postdoctoral fellowship.
\end{ack}

\bibliographystyle{plainnat}
\bibliography{bibtex/biblio}

\appendix

\crefalias{section}{appendix}

\newpage

\section{Further Examples of \ac{LRM} Schemes}
\label{app:examples}
\begin{example}\label{ex:SGLD}
  The classic \emph{Stochastic Gradient Langevin Dynamics} 
  \citep{wellingBayesianLearningStochastic2011} iterates as
  \begin{equation}\label{eq:sgld}\tag{SGLD}
    \dpt_{k+1} = \dpt_k - \step_{k+1} \stochgrad f(\dpt_k) + \sqrt{2\step_{k+1}}\,\stdgauss_{k+1},
  \end{equation}
  where $\stochgrad f$ is the gradient of the negative log-likelihood of a random batch of the data.
  \eqref{eq:sgld} fits the \ac{LRM} template by setting
  $\noise_{k+1} \coloneqq \stochgrad f(x_k) - \grad f(x_k)$, and $\bias_{k+1} \coloneqq 0$.$\hfill\qqed$
\end{example}

\begin{example}\label{ex:PLA}
  The \emph{Proximal Langevin Algorithm}
  \citep{pereyra2016proximal,bernton2018langevin,wibisono2019proximal} is defined via
  \begin{equation}\label{eq:pla}\tag{PLA}
    \dpt_{k+1} = \dpt_k - \step_{k+1} \grad f(\dpt_{k+1}) + \sqrt{2\step_{k+1}}\,\stdgauss_{k+1}.
  \end{equation}
  This algorithm is implicit, and it is assumed that one can solve \eqref{eq:pla} for $\dpt_{k+1}$.
  By setting 
  $\bias_{k+1} \coloneqq \grad f(\dpt_{k+1}) - \grad f(\dpt_k)$ and $\noise_{k+1} \coloneqq 0$, we
  see that this algorithm also follows the \ac{LRM} template.
  $\hfill\qqed$
\end{example}

\section{Additional Related Work}
\label{app:comparison}

Our paper studies the behavior of a wide range of Langevin-based sampling algorithms proposed in the literature in the asymptotic
setting under minimal assumptions. This allows us to give last-iterate guarantees in Wasserstein distance. As stressed in \cref{sec:introduction}, our goal is \emph{not} to provide non-asymptotic rates in this general setting as the problem is
inherently NP-Hard. However, given more assumptions and structures on the potential $f$, there is a
plethora of works which prove convergence rates for the last iterates in Wasserstein distance. In this appendix, we provide additional backgraound for these works and the methods used in the literature.

A powerful framework for quantifying the global discretization error of a numerical algorithm is the
mean-square analysis framework \citep{milstein2004stochastic}. This framework furnishes a general recipe
for controlling short and long-term integration errors. For sampling, this framework has been
applied to prove convergence rates for Langevin Monte-Carlo (the Euler-Maruyama discretization of
\eqref{eq:LD}) in the strongly-convex setting
\citep{liStochasticRungeKuttaAccelerates2020,li2022sqrt}. Similar to our work, the convergence obtained in these works is
last-iterate and in Wasserstein distance. One of the essential ingredients in the latter work is the
contraction property of the SDE, which is ensured by the strong convexity assumption. This, in turn, implies strong
non-asymptotic convergence guarantees.

It is an interesting future direction to study the combination of the Mean-Squared analysis together
with the Picard process and its applicability to more sophisticated algorithms (such as \ac{LRM} schemes with
bias and noise), as well as non-convex potentials.

As explained in \cref{sec:intuition}, one of the main themes in proving error bounds for sampling is the natural
relation between sampling and optimization in the Wasserstein space. This point of view, when applied to strongly-convex potentials, has produced numerous non-asymptotic guarantees; see \citep{dalalyan2017user, chewi2023log} for a recent account and the references therein. Note that strong convexity is crucial for the analysis used in the aforementioned work. Moreover, the error bounds for biased and noisy discretizations do \emph{not} decrease
with the step-size or iteration count; see \citep[Theorem 4, Eqn. (14)]{dalalyan2017user}. This means that while the
bound is non-asymptotic, it does not automatically result in an asymptotic convergence. Finally, we stress that these approaches are orthogonal to our techniques: We view a sampling algorithm as a
(noisy and biased) discretization of a dynamical system (and not necessarily a gradient flow), and
use tools from dynamical system theory to provide asymptotic convergence results.

\section{Proofs for \texorpdfstring{\cref{sec:theory-statement}}{Section 4}}
\label{app:proofs-lemmas}

\subsection{Proof of \texorpdfstring{\cref{thm:APT}}{Theorem 1}}

In this appendix, we bring the detailed proof of \cref{thm:APT}. Recall that we interpolate the
iterates of the \ac{LRM} scheme $\crl{x_k}$ as
\begin{equation}\tag{\ref{eq:interpolation}}
  \cpt_t = \dpt_k + \prn*{t- \tau_k}\crl{v(\dpt_k) + \exof{\pert_{k+1} \given \Fil_{t}}} +
  \sigma(\dpt_k)\,(B_t - B_{\tau_k}).
\end{equation}
Moreover, for a fixed $t > 0$, we considered the Brownian motion $B^{(t)}_s = B_{t+s} - B_t$, and
constructed two important processes: the \meandynamics defined via
\begin{equation}\label{eq:mean-sde}
  \dd\Xfl{s}{t} = v(\Xfl{s}{t})\dd s + \sigma(\Xfl{s}{t})\dd B_s^{(t)},\quad \Xfl{0}{t} = X_t,
\end{equation}
and the Picard process \eqref{eq:picard-sde} constructed as
\begin{equation}\tag{\ref{eq:picard-sde}}
  \Xpi{s}{t} = X_t + \int_0^s v(\cpt_{t+u})\dd u + \int_0^s \sigma(\cpt_{t+u})\dd B_u^{(t)}.
\end{equation}

Let us fix $T > 0$, and for $s \in [0, T]$ decompose the distance between the interpolation and the
\meandynamics as
\begin{equation}\tag{\ref{eq:main-decomp}}
  \tfrac{1}{2}\nrm{\cpt_{t + s} - \Xfl{s}{t}}^2 
  \leq \nrm{\Xpi{s}{t} - \Xfl{s}{t}}^2 + \nrm{\cpt_{t+s} - \Xpi{s}{t}}^2,
\end{equation}
where we have used $\nrm{a + b}^2 \leq 2\nrm{a}^2 + 2\nrm{b}^2$. We now bound each term of this
decomposition. Notice that due to the synchronous coupling of the processes, the Brownian motion cancels
out in the differences. 

The first term controls how close the Picard process is to the \meandynamics, and is bounded in the
following lemma.

\begin{lemma}\label{lem:bound-dist-picard-mean}
  For fixed $t, T > 0$ and $0 \leq s \leq T$, the distance of the Picard process and the
  \meandynamics is bounded as
  \[
    \nrm{\Xpi{s}{t} - \Xfl{s}{t}}^2 \leq 2(T+1)L^2 \int_0^s \nrm{\Xfl{u}{t} - \cpt_{t+u}}^2 \,\dd u.
  \]
\end{lemma}
\begin{proof}
  By the auxiliary \cref{lem:cauchy} below, Lipschitzness of $v,\sigma$, It\^{o} isometry (see, \eg
  \cite{zhang20a}) and $s\leq T$, we have
  \begin{align*}
    \E \nrm{\Xpi{s}{t} - \Xfl{s}{t}}^2 
       &= \E\nrm*{\int_0^s v(\Xfl{u}{t}) - v(\cpt_{t+u}) \,\dd u + \int_0^s \sigma(\Xfl{u}{t}) -
       \sigma(\cpt_{t+u}) \,\dd B_u^{(t)} }^2 \\
       &\leq 2s \int_0^s \E \nrm*{v(\Xfl{u}{t}) - v(\cpt_{t+u})}^2 \,\dd u + 2\E \int_0^s
       \nrm*{\sigma(\cpt_{t+u}) - \sigma(\Xfl{u}{t})}_F^2\dd u \\
       &\leq 2(T + 1)L^2 \int_0^s \E\nrm{\Xfl{u}{t} - \cpt_{t+u}}^2 \,\dd u.\qedhere
  \end{align*}
\end{proof}

For the rest of the proof, we need to define the continuous-time piecewise-constant processes
$\Xbar(\tau_k + s) = \cpt_k$, $\gammabar(\tau_k + s) = \step_{k+1}$, $\pertbar(\tau_k + s) =
\pert_{k+1}$, and 
$\pert(\tau_k + s) = \exof{\pert_{k+1} \given \Fil_{\tau_k + s}}$, for $0 \leq s < \step_{k+1}$.
Also, let $m(t) = \sup\{ k \geq 0 : \tau_k \leq t\}$ so that $\tau_{m(t)} \leq t < \tau_{m(t) + 1}$.

To bound the second term in \eqref{eq:main-decomp}, we have seen that 
\begin{align*}
  \cpt_{t+s} - \Xpi{s}{t}  &= \int_t^{t+s} v(\Xbar(u))\dd u  - \int_0^s v(X_{t+u}) \dd u \\
    & + \int_t^{t+s} \sigma(\Xbar(u))\dd B_u   - \int_0^s \sigma(X_{t+u}) \dd B_u^{(t)} \\ 
    & + \Delta_\pert(t,s),
\end{align*}
where $\Delta_\pert(t,s)$ plays the role of accumulated noise and bias from time $t$ to $t+s$, and is
defined as
\begin{equation}\label{eq:delta-definition}
    \Delta_\pert(t, s) \coloneqq \sum_{i=n}^{k-1} \step_{i+1}\pert_{i+1} + 
    (t+s - \tau_k)\exof{\pert_{k+1} \given \Fil_{t+s}} 
    - (t - \tau_n)\exof{\pert_{n+1} \given \Fil_{t}},
\end{equation}
with $k = m(t+s)$ and $n = m(t)$. We therefore have
\begin{align}
  \E\nrm{\cpt_{t+s} - \Xpi{s}{t}}^2 
  & \leq 3\E\nrm*{\int_t^{t+s} v(\cpt_{u}) - v(\Xbar(u))\dd u }^2 \nonumber \\
  &\quad + 3\E\nrm*{\int_t^{t+s} \sigma(\cpt_u) - \sigma(\Xbar(u))\dd B_u }^2 
  + 3\E\nrm*{\Delta_\pert(t,s)}^2 \nonumber \\
  & \leq 3s\int_t^{t+s} \E\nrm*{v(\cpt_{u}) - v(\Xbar(u))}^2\dd u \nonumber \\
    &\quad + 3\E\int_t^{t+s} \nrm*{\sigma(\cpt_{u}) - \sigma(\Xbar(u))}_F^2\dd u 
    + 3\E\nrm*{\Delta_\pert(t,s)}^2 \nonumber \\
  & \leq 3(s+1)L^2 \int_t^{t+s} \E\nrm{\cpt_{u} - \Xbar(u)}^2 \dd u
    + 3\E\nrm*{\Delta_\pert(t,s)}^2. \label{eq:bound-integral-dist-picard-interp}
\end{align}
For bounding the term inside the integral, we have 
\begin{align*}
  \E\nrm{X_u - \Xbar(u)}^2 
  &= \E\nrm{(u - \tau_{m(u)})\crl{v(\Xbar(u)) + \pert(u)} + \sigma(\Xbar(u))\,(B_u - B_{\tau_{m(u)}}) }^2 \\
  & \leq 4\gammabar(u)^2\prn*{\E\nrm{v(\Xbar(u))}^2 + \E\nrm{\pert(u)}^2 } +
  2\gammabar(u)\,\E\,\tr\prn*{\sigma(\Xbar(u))^\top \sigma(\Xbar(u))}.
\end{align*}
We have used the fact that
\begin{align*}
  \E\nrm{\sigma(\Xbar(u))\,(B_u - B_{\tau_{m(u)}})}^2
  &=\E\prn*{(B_u - B_{\tau_{m(u)}})^\top \sigma(\Xbar(u))^\top\sigma(\Xbar(u))\,(B_u - B_{\tau_{m(u)}})} \\
  &=\E\tr\prn*{\sigma(\Xbar(u))^\top\sigma(\Xbar(u))\,(B_u - B_{\tau_{m(u)}})(B_u - B_{\tau_{m(u)}})^\top } \\
  &=\E\brk*{\exof{\tr\prn{\sigma(\Xbar(u))^\top\sigma(\Xbar(u))\,(B_u - B_{\tau_{m(u)}})(B_u -
  B_{\tau_{m(u)}})^\top } \given \Fil_{\tau_{m(u)}}}} \\
  &=(u - \tau_{m(u)})\E\brk*{\tr\prn{\sigma(\Xbar(u))^\top\sigma(\Xbar(u))}}
\end{align*}

Notice that since conditional expectation is a projection in $L^2$, we have $\E\nrm{Z(u)}^2 \leq
\E\nrm{\pertbar(u)}^2$. Using this fact, along with boundedness of $\sigma(\cdot)$ by $C_\sigma$, and
\cref{lem:gradient-bound-l2} we get
\begin{align*}
  \E\brk*{\nrm{X_u - \Xbar(u)}^2} %
  & \leq 4\gammabar(u)^2\prn*{\E\nrm{v(\Xbar(u))}^2 + \E\nrm{\pertbar(u)}^2 } +
  2\gammabar(u)\,\E\,\tr\prn*{\sigma(\Xbar(u))^\top \sigma(\Xbar(u))} \\
    &\leq 4\gammabar(u)^2\E\nrm{v(\Xbar(u))}^2 + 8\gammabar(u)^2\sigma^2 +
    4\gammabar(u)^2\bigoh(\gammabar(u)) + 2C_\sigma\gammabar(u) \leq C\gammabar(u),
\end{align*}
for some constant $C > 0$.
Plugging this estimate into \eqref{eq:bound-integral-dist-picard-interp} after taking expectation
yields
\begin{align*}
  \E\brk*{\nrm{\cpt_{t+s} - \Xpi{s}{t}}^2} & \leq 3(s+1)L^2C \int_t^{t+s} \gammabar(u)\dd u + 3\E\nrm{\Delta_\pert(t, s)}^2 \\
    & \leq 3(s + 1)sL^2C \sup_{u \in [t, t+s]} \gammabar(u) + 3\E\nrm{\Delta_\pert(t, s)}^2 \\
    & \leq 3(T + 1)^2L^2C \sup_{u \in [t, t+T]} \gammabar(u) + 3\sup_{u \in [0, T]}\E\nrm{\Delta_\pert(t, u)}^2 
\end{align*}
Taking supremum over $s \in [0, T]$ and noticing that the right-hand-side is independent of $s$ and
$\step_k \to 0$, together with \cref{lem:sup-bound-bias-noise-cumulative} yields
\begin{align}
  A_t &\coloneqq \sup_{0 \leq s \leq T} \E\brk*{\nrm{\cpt_{t+s} - \Xpi{s}{t}}^2}
  \label{eq:definition-A-t} \\
  &\leq 3(T + 1)^2L^2C \sup_{t\leq u \leq t+T} \gammabar(u) 
  + 3\sup_{0 \leq u \leq T} \E\brk*{\nrm{\Delta_\pert(t, u)}^2} \nonumber \\
  &\to 0 \quad \text{as } t\to\infty, \nonumber
\end{align}
showing that the Picard process gets arbitrary close to the original interpolation, as $t\to\infty$.

Let us return to the decomposition \eqref{eq:main-decomp}. By taking expectation and using
\eqref{eq:bound-gronwall} and \eqref{eq:definition-A-t} we obtain
\begin{align*}
  \E\brk*{\nrm{\cpt_{t + s} - \Xfl{s}{t}}^2 } & \leq 2(T + 1) L^2 \int_0^s \E\brk*{\nrm{\cpt_{t+u} - \Xfl{u}{t}}^2}\dd u + 2 A_t \\
    &\leq 2A_t \exp\prn*{ s (T + 1)L^2 } \\
    &\leq 2A_t \exp((T + 1)^2L^2),
\end{align*}
where in the last line we have used the Gr\"onwall lemma. Thus,
\[
  \lim_{t \to \infty} \sup_{s\in [0, T]} \E\brk*{ \nrm{X_{t + s} - \Xfl{s}{t}}^2 } = 0.
\]
Recall that the Wasserstein distance between $X_{t+s}$ and $\Xfl{s}{t}$ is the infimum over all
possible couplings between them, having the correct marginals. As $\Xfl{s}{t}$ has the same marginal
as the Langevin diffusion started from $\cpt_t$ at time $s$, and the synchronous coupling of the
interpolation and the \meandynamics produces a specific coupling between them, we directly get
\[
  \Wtwo(\cpt_{t+s}, \Xfl{s}{t}) \leq \E\brk*{ \nrm{\cpt_{t + s} -
  \Xfl{s}{t}}^2 }^{\frac{1}{2}},
\]
which implies 
\[
  \lim_{t \to \infty} \sup_{s\in [0, T]} \Wtwo(X_{t + s} , \Xfl{s}{t}) = 0,
\]
as desired.\hfill\qed

\subsection{Auxiliary Lemmas}

\lembounddeltaZ*
\begin{proof}
  Define $\Delta_b$ and $\Delta_\noise$ the same way as in \eqref{eq:delta-definition}.
  By Cauchy-Schwarz we have
  \begin{align*}
    &\nrm{\Delta_b(t, s)}^2 \\
    &\hspace{2mm}\leq \prn*{\sum_{i=n}^{k-1} \step_{i+1}\nrm{\bias_{i+1}} 
      + (t+s - \tau_k)\nrm{\exof{\bias_{k+1} \given \Fil_{t+s}}} 
    + (t - \tau_n)\nrm{\exof{\bias_{n+1} \given \Fil_{t}}}}^2 \\
    &\hspace{2mm} \leq \prn*{2\step_{n+1} + s}
    \prn*{\sum_{i=n}^{k-1} \step_{i+1}\nrm{\bias_{i+1}}^2
      + (t+s - \tau_k)\nrm{\exof{\bias_{k+1} \given \Fil_{t+s}}}^2
    + (t - \tau_n)\nrm{\exof{\bias_{n+1} \given \Fil_{t}}}^2},
  \end{align*}
where the last inequality comes from $\sum_{i=n}^{k-1} \step_{i+1} \leq s$, $t+s - \tau_k \leq
\step_{k+1}$, $t - \tau_n \leq \step_{n+1}$, and $\step_{k+1} \leq \step_{n+1}$.

  Noticing that conditional expectation is a contraction in $L^2$ and letting $k' =
  m(t+T)$, we get
  \[
    \sup_{0\leq s \leq T} \E\brk*{ \nrm{\Delta_b(t, s)}^2 } \leq (2 + T)
    \prn*{\sum_{i=n}^{k'-1} \step_{i+1}\E\nrm{\bias_{i+1}}^2
      + \sup_{n\leq j \leq k'+1}\step_{j+1}\E\nrm{\bias_{j+1}}^2
    + \step_{n+1}\E\nrm{\bias_{n+1}}^2}
  \]
  Now, invoking \cref{lem:gradient-bound-l2} yields
  \begin{align*}
    \sup_{0\leq s \leq T} \E\brk*{ \nrm{\Delta_b(t, s)}^2 } &\leq 
    C(2 + T)
    \prn*{\sum_{i=n}^{k'-1} \step_{i+1}^2 + \sup_{n\leq j \leq k'+1}\step_{j+1}^2 + \step_{n+1}^2}
    \\
    &\leq C(2 + T)\prn*{\sum_{i=n}^{k'-1} \step_{i+1}^2 + 2\step_{n+1}^2}\\
    &\leq C(2 + T)\prn*{T + 2\step_{n+1}}\sup_{0\leq s \leq T} \gammabar(t+s).
  \end{align*}
  As $t \to \infty$, the last quantity vanishes, since $\step_n \to 0$.

  For the noise we have 
  \begin{align*}
    \nrm*{\Delta_\noise(t,s)}^2 &\leq 2\nrm*{\sum_{i=n}^{k-1} \step_{i+1}\noise_{i+1}}^2
    + 4\nrm{(t+s - \tau_k)\exof{\noise_{k+1} \given \Fil_{t+s}}}^2
    + 4\nrm{(t-\tau_n)\exof{\noise_{n+1} \given \Fil_t}}^2 \\
    & \leq 2\nrm*{\sum_{i=n}^{k-1} \step_{i+1}\noise_{i+1}}^2 + 4\step_{k+1}^2 \nrm{\noise_{k+1}}^2
    + 4 \step_{n+1}^2\nrm{\noise_{n+1}}^2.
  \end{align*}
  Taking expectations and then sup, we get
  \begin{align*}
    \sup_{0\leq s\leq T} \E\brk*{\nrm*{\Delta_\noise(t,s)}^2} 
    & \leq 2 \sup_{n+1 \leq k \leq m(t+T)} \E\nrm*{\sum_{i=n}^{k-1} \step_{i+1}\noise_{i+1}}^2 + 4\step_{k+1}^2 \sigma^2
    + 4 \step_{n+1}^2\sigma^2.
  \end{align*}
  Since $\crl{\noise_i}$ is a martingale difference sequence, we have that $\crl*{\sum_{i=n}^{k-1}
  \step_{i+1}\noise_{i+1}}_{k > n}$ is a martingale. Thus, by the boundedness of the second moments
  of $\noise_i$, we get
  \[
    \E\nrm*{\sum_{i=n}^{k-1} \step_{i+1}\noise_{i+1}}^2 
      = \sum_{i=n}^{k-1} \step_{i+1}^2\E\nrm*{\noise_{i+1}}^2
      \leq \sigma^2 \sum_{i=n}^{k-1} \step_{i+1}^2.
  \]
  Hence,
  \[
    \lim_{n\to\infty} 
    \sup\crl*{\E\nrm{\sum_{i=n}^{k-1}\step_{i+1}\noise_{i+1}}^2 :  n < k \leq m(\tau_n + T) }
    \leq \lim_{n\to\infty} \sigma^2 \sum_{i=n}^{\infty} \step_{i+1}^2 = 0.
  \]
\end{proof}

\gradientboundltwo*
\begin{proof}
Without loss of generality, suppose $v$ has a stationary point at $0$. We repeatedly use the fact
that $\E \nrm{v(\dpt_k)}^2 \leq L^2 \E \nrm{\dpt_k}^2$. Moreover, by \cref{ass:f} we have
$\inner{v(x), x} \leq C_v(\nrm{x} + 1)$, and $\nrm{\sigma(x)}_F^2 \leq C_\sigma$.

Define $a_k \coloneqq \E\nrm{\dpt_k}^2$. We have
\begin{align}
  a_{k+1} - a_k &=
   \step_{k+1}^2 \E \nrm{v(\dpt_k) + \pert_{k+1}}^2
   + \step_{k+1} \E \nrm{\sigma(\dpt_k)\stdgauss_{k+1}}^2
   + 2\step_{k+1}\E \inner{\dpt_k, v(\dpt_k) + \pert_{k+1}} \nonumber \\
   &\hspace{1cm}
   + 2\step_{k+1}^{1/2}\E \inner{\dpt_k, \sigma(\dpt_k)\stdgauss_{k+1}} 
   + 2\step_{k+1}^{3/2}\E \inner{v(\dpt_k) + \pert_{k+1}, \sigma(\dpt_k)\stdgauss_{k+1}} \nonumber
   \\
   &\leq 2L^2\step_{k+1}^2a_k 
   + 2\step_{k+1}^2\E \nrm{\pert_{k+1}}^2
   + \step_{k+1} C_\sigma 
   + 2\step_{k+1} C_v (\sqrt{a_k} + 1)
   + 2\step_{k+1}\sqrt{a_k}\sqrt{\E\nrm{\pert_{k+1}}^2} \nonumber \\
   &\hspace{1cm} 
   + 2\step_{k+1}^{3/2}\sqrt{C_\sigma}\sqrt{\E\nrm{\pert_{k+1}}^2}
   \label{eq:initial-bound-an-diff}
\end{align}

  By \cref{ass:noise-bias}, there is some $C_b > 0$ such that
  $\E\nrm{\bias_{k+1}}^2 \leq C_b(\step_{k+1}^2a_k + \step_{k+1})$, and we have
\begin{equation}\label{eq:bound-on-second-moment-bias}
  \E\nrm{\pert_{k+1}}^2 \leq 2\E\nrm{\bias_{k+1}}^2 + 2\E\nrm{\noise_{k+1}}^2 \leq 
  2C_b(\step_{k+1}^2a_k + \step_{k+1}) + 2\sigma^2.
\end{equation}
Moreover, as $\sqrt{p+q}\leq \sqrt{p} +\sqrt{q}$, we have
\begin{equation}\label{eq:bound-on-sqrt-second-moment-bias}
  \sqrt{\E\nrm{\pert_{k+1}}^2} \leq \sqrt{2C_b}(\step_{k+1}\sqrt{a_k} + \sqrt{\step_{k+1}}) + \sqrt{2}\sigma.
\end{equation}
Plugging the bounds from \eqref{eq:bound-on-second-moment-bias} and 
\eqref{eq:bound-on-sqrt-second-moment-bias} into \eqref{eq:initial-bound-an-diff} gives
\begin{equation}\label{eq:recursion-for-a-n}
  \begin{aligned}
    a_{k+1} - a_k &\leq 
    2L^2\step_{k+1}^2a_k + 4C_b\step_{k+1}^4 a_k + 4C_b\step_{k+1}^3 + 4\step_{k+1}^2\sigma^2 \\
    &\hspace{5mm} + \step_{k+1} C_\sigma 
    + 2\step_{k+1} C_v \sqrt{a_k} + 2\step_{k+1} C_v  \\
    &\hspace{5mm}
    + 2\sqrt{2C_b}\step_{k+1}^2 a_k + 2\sqrt{2C_b}\step_{k+1}^{3/2}\sqrt{a_k} +
    2\sqrt{2}\sigma\step_{k+1}\sqrt{a_k} \\
    &\hspace{5mm} + 2\sqrt{2C_bC_\sigma}\step_{k+1}^{5/2}\sqrt{a_k} +
    2\sqrt{2C_bC_\sigma}\step_{k+1}^2 + 2\step_{k+1}^{3/2}\sqrt{2C_\sigma}\sigma \\
    &\eqqcolon P\step_{k+1}^2\,a_k + Q\step_{k+1}\sqrt{a_k} + R\step_{k+1},
  \end{aligned}
\end{equation}
where
\begin{align*}
  P &= 2L^2 + 4C_b\step_{k+1}^2 + 2\sqrt{2C_b} \\
  Q &= 2C_v + 2\sqrt{2C_b}\sqrt{\step_{k+1}} + 2\sqrt{2}\sigma + 2\sqrt{2C_b}\step_{k+1} +
  2\sqrt{2C_bC_\sigma}\step_{k+1}^{3/2} \\
  R &= 4C_b\step_{k+1}^2 + 4\step_{k+1}\sigma^2 +  C_\sigma + 2C_v + 2\sqrt{2C_bC_\sigma}\step_{k+1}
  + 2\step_{k+1}^{1/2}\sqrt{2C_\sigma}\sigma.
\end{align*}
The exact values of $P$, $Q$, and $R$ are irrelevant, and we only need upper
bounds for them. Assuming that $\gamma_{k+1} < 1$ for all $k$, we replace the
three quantities by
  \begin{equation}\label{eq:definition-p}
  \begin{aligned}
    P &= 2L^2 + 4C_b + 2\sqrt{2C_b}  \\
    Q &= 2C_v + 2\sqrt{2C_b} + 2\sqrt{2}\sigma + 2\sqrt{2C_b} + 2\sqrt{2C_bC_\sigma} \\
    R &= 4C_b + 4\sigma^2 + C_\sigma + 2C_v + 2\sqrt{2C_bC_\sigma} + 2\sqrt{2C_\sigma}\sigma.
  \end{aligned}
  \end{equation}

  Now, define $h_k = \step_{k+1}^2 a_k$. The recursion
  \eqref{eq:recursion-for-a-n} in terms of $h_k$ becomes
\begin{align*}
  h_{k+1} \leq h_k(1 + P\step_{k+1}^2)\frac{\step_{k+2}^2}{\step_{k+1}^2} 
  + \sqrt{h_k} Q\step_{k+2}^2 + R\step_{k+1}\step_{k+2}^2.
\end{align*}
We now prove that there exists some $M > 0$ so that $h_k \leq M\step_{k+1}$ by induction. Suppose
it is the case for $k$, and we prove it for $k+1$. Using the induction hypothesis we get
\begin{align*}
  h_{k+1} &\leq M\step_{k+1}(1 + P\step_{k+1}^2)\frac{\step_{k+2}^2}{\step_{k+1}^2} 
  + \sqrt{M\step_{k+1}} Q\step_{k+2}^2 + R\step_{k+1}\step_{k+2}^2  \\
  & = M(1 + P\step_{k+1}^2)\frac{\step_{k+2}^2}{\step_{k+1}} 
+ \sqrt{M}\, Q \sqrt{\step_{k+1}} \step_{k+2}^2 + R\step_{k+1}\step_{k+2}^2
\end{align*}
For the last to be less than $M\step_{k+2}$, we have to verify
\[
  M(1 + P\step_{k+1}^2)\frac{\step_{k+2}}{\step_{k+1}} 
+ \sqrt{M}Q\sqrt{\step_{k+1}}\step_{k+2} + R\step_{k+1}\step_{k+2} 
  \leq M
\]
or equivalently,
\[
  M\prn*{\frac{\step_{k+2}}{\step_{k+1}} + P\step_{k+1}\step_{k+2} - 1}
+ \sqrt{M} Q\sqrt{\step_{k+1}}\step_{k+2} + R\step_{k+1}\step_{k+2} 
  \leq 0.
\]
This is a quadratic equation in $\sqrt{M}$, and for this inequality to hold, we prove that the
leading coefficient is negative, and the largest root is bounded above by some constant not
depending on $n$.

Negativity of the leading coefficient is equivalent to 
\[
  \frac{\step_{k+2}}{\step_{k+1}} + P\step_{k+1}\step_{k+2} < 1,
\]
which is implied by our assumption on the step size.

The larger root of the equation is 
\begin{align*}
  &\frac{\prn*{-4 \step_{k+1}^2 \step_{k+2}^2 P R + \step_{k+1}\step_{k+2}(\step_{k+2} Q^2+4 R)-4 R
    \step_{k+2}^2}^{1/2} +\sqrt{\step_{k+1}}\step_{k+2}Q}{2 (1 - \step_{k+1} \step_{k+2} P -
    \step_{k+2}/\step_{k+1})} \\
    &\hspace{5mm} < \frac{\sqrt{\step_{k+1}}\step_{k+2}Q + \sqrt{R\step_{k+1}\step_{k+2}}}%
    {(1 -\step_{k+1} \step_{k+2} P - \step_{k+2}/\step_{k+1})} \\
  &\hspace{5mm} \leq \frac{\sqrt{\step_{k+1}}\step_{k+1}Q + \sqrt{R}\step_{k+1}}%
    {(1 -\step_{k+1} \step_{k+2} P - \step_{k+2}/\step_{k+1})}.
\end{align*}
By our assumption on the step size that
\[
  \frac{\step_{k+2}}{\step_{k+1}} + P\step_{k+1}\step_{k+2} < 1 - \step_{k+1},
\]
we get that the larger root is smaller than
\[
  \frac{\sqrt{\step_{k+1}}\step_{k+1}Q + \sqrt{R}\step_{k+1}}%
  {\step_{k+1}} = \sqrt{\step_{k+1}}Q + \sqrt{R} < Q + \sqrt{R}.
\]
Letting $M := Q + \sqrt{R}$ gives the desired result.

The second argument of the lemma follows from \cref{ass:noise-bias} and the first
result of the lemma.
\end{proof}

\begin{lemma}\label{lem:cauchy}
For a vector valued function $g \in L^2(\R;\R^d)$, one has
\[
  \nrm*{\int_0^s g(u)\,du}^2 \leq 
    \prn*{\int_0^s \nrm{g(u)}\,du}^2 \leq s \int_0^s \nrm{g(u)}^2 \,du.
\]
\end{lemma}

\section{Proofs for \texorpdfstring{\cref{sec:stability}}{Section 5}}
\label{app:stability}
\subsection{Proof of \texorpdfstring{\cref{thm:LRMstability}}{Theorem 3}}

For brevity, let us write $\Fil_k$ instead of $\Fil_{\tau_k}$.
Opening up $\nrm{x_{k+1}}^2 = \nrm{ x_k + \step_{k+1} \crl*{ v(x_k) + \pert_{k+1} } +
\sqrt{\step_{k+1}}\sigma(x_k)\,\xi_{k+1} }^2$ and ignoring every term that is zero-mean under $\exof{\cdot
\given \Fil_k}$, we get
\begin{align}
\nonumber
\Ekc &= \E\Big[\nrm{x_k}^2 + 2\step_{k+1}\inner{x_k, v(x_k) + Z_{k+1} }  \\
& \hspace{6mm} + \step_{k+1}^2 \nrm{ v(x_k) + \pert_{k+1}}^2  + \step_{k+1}
  \nrm{\sigma(x_k)\xi_{k+1}}^2 + 2\step_{k+1}^{\frac{3}{2}}\inner{\sigma(x_k)\xi_{k+1}, \bias_{k+1}}  \,\big\vert\, \mathcal{F}_k\Big] \nonumber \\
&\leq \nrm{x_k}^2 + 2\step_{k+1}\prn*{ \inner{x_k, v(x_k)} + C_\sigma /2 } + 2\step_{k+1}^2 \nrm{v(x_k)}^2 \nonumber \\
&\hspace{6mm}+  \E\brk*{2\step_{k+1}^2 \nrm{\pert_{k+1}}^2 + 2\step_{k+1} \inner{x_k, \pert_{k+1}} +
2\step_{k+1}^{\frac{3}{2}}\inner{\sigma(x_k)\xi_{k+1}, \bias_{k+1}} \,\big\vert\, \mathcal{F}_k} \nonumber\\
&\leq \nrm{x_k}^2 + 2\step_{k+1}\prn*{ \inner{x_k, v(x_k)} + C_\sigma/2 +
\step_{k+1}^{\frac{1}{2}} C_\sigma / 4} + 2\step_{k+1}^2 \nrm{v(x_k)}^2   \label{eq:hold}\\
\nonumber
&\hspace{10mm}+  \E\brk*{2\step_{k+1}^2 \nrm{\pert_{k+1}}^2\vert \mathcal{F}_k}   + \step_{k+1}^{\frac{3}{2}} \E\brk*{\nrm{\bias_{k+1}}^2\vert \mathcal{F}_k}  + 2 \E\brk*{\step_{k+1} \inner{x_k, \bias_{k+1}} \vert \mathcal{F}_k}.
\end{align}
Recalling \eqref{eq:bias-assumption} in \cref{ass:noise-bias}, we have for some $C>0$
\begin{equation}
  \E \nrm{\pert_{k+1}}^2 \leq 2 \sigma^2 + 2C \prn*{\step_{k+1}^2 \E\nrm{v(x_k)}^2 + \step_{k+1} }
\end{equation}
Without loss of generality, assume $\step_k \leq 1$ and $\Ek \geq 1$ (so that $\prn*{\Ek}^2 \geq
\Ek$) for all $k$. Then, $\nrm{v(x_k)}^2 \leq L^2 \nrm{x_k}^2$, together with
\cref{asm:dissipativity} and the Cauchy-Schwartz inequality on the last term of \eqref{eq:hold}, implies
\begin{align}
\Ek[k+1] &\leq \Ek - 2\alpha \step_{k+1} \Ek  + 2\step_{k+1}\left(\beta + C_\sigma +
\frac{1}{2}\step_{k+1}^{\frac{1}{2}}C_\sigma \right) + 2L^2 \step_{k+1}^2 \Ek \nonumber
\\ &\hspace{15mm}  + 2\step_{k+1}^2 \brk*{2\sigma^2+2C\prn*{L^2\step_{k+1}^2 \Ek + \step_{k+1} }}\nonumber 
\\ &\hspace{15mm}  + \step_{k+1}^{\frac{3}{2}}C \prn*{L^2\step_{k+1}^2 \Ek + \step_{k+1} }\nonumber 
\\ &\hspace{15mm}   +2 \step_{k+1}\sqrt{C}\sqrt{ L^2\step_{k+1}^2 \left(\Ek\right)^2 + \step_{k+1} \Ek } \nonumber \\
&\leq \Ek(1- C_1\step_{k+1} + C_2\step_{k+1}^{\frac{3}{2}} ) + C_3\step_{k+1}  \nonumber
\end{align}
for some constants $C_1,C_2,C_3$ depending on $L, C, \sigma, \alpha, \beta$, and $d$. Since $\step_k\to 0$, there exist $\tilde{\alpha},\tilde{\beta} >0$ and $k_0$ such that, for all $k \geq k_0$, 
\begin{align*}
  \Ek[k+1] &\leq \Ek(1- \tilde{\alpha}\step_{k+1}) + \tilde{\beta}\step_{k+1}, \quad 1- \tilde{\alpha}\step_{k+1} > 0.
\end{align*}
A simple induction yields 
\begin{align*}
\sup_k\ \Ek \leq \max\left\{ \frac{\tilde{\beta}}{\tilde{\alpha}} ,  \Ek[k_0]  \right\}
\end{align*}which concludes the proof.\hfill\qed

\subsection{Proof of \texorpdfstring{\cref{thm:LRMstability2}}{Theorem 4} for Constant Diffusion}%
\label{sec:proof-weak-dissipativity-ld}

Before proceeding, we need a lemma which can be distilled from \cite[Proposition 8]{durmus2017nonasymptotic}:
\begin{lemma}
\label{lem:helper}
Suppose $\grad f$ is $L$-Lipschitz. Fix $x\in\R^d$ and $\step>0$, let $\tilde{x}^+ = x- \step\nabla f(x) + \sqrt{2\step}\xi$. Then
\begin{align}
\E \brk*{\exp\prn*{  \frac{1}{2}\inner{\nabla f(x), \tilde{x}^+ -x } + \frac{L}{4} \nrm{\tilde{x}^+ -x}^2   } }
\leq (1-\step L)^{-d/2}e^{-\frac{\step}{4}\nrm{\nabla f(x)}^2}.
\end{align}
\end{lemma}

Let $\tilde{x}_{k+1} \coloneqq x_k- \step_{k+1} \nabla f(x_k)  + \sqrt{2\step_{k+1}}\,\xi_{k+1} $ so that $x_{k+1} - x_k = \tilde{x}_{k+1} - x_k - \step_{k+1}\prn*{  \noise_{k+1}+\bias_{k+1} }$. Conditioned on $x_k, \noise_{k+1}, \noise'_{k+1}, \xi'_{k+1}$, and using the $L$-Lipschitzness of $\nabla f$, we get
\begin{align}
\nonumber
&e^{-\frac{1}{2}f(x_k)} \E e^{\frac{1}{2}f(x_{k+1})} \\
&\hspace{4mm}\leq \E \exp\prn*{\frac{1}{2}\inner{\nabla f(x_k), x_{k+1}-x_k} + \frac{L}{4}\nrm{x_{k+1}-x_k}^2} \\
&\hspace{4mm}\leq \E\exp\Bigg\{
     \frac{1}{2} \inner{\nabla f(x_k), \tilde{x}_{k+1}-x_k} - \frac{1}{2} \inner{ \nabla f(x_k), \step_{k+1} \noise_{k+1} } \\
       &\hspace{20mm}  - \frac{1}{2} \inner{ \nabla f(x_k), \step_{k+1} \bias_{k+1} }  + \frac{L}{2} \nrm{\tilde{x}_{k+1} -x_k}^2  + {L\step_{k+1}^2}\nrm{\noise_{k+1}}^2+ {L\step_{k+1}^2}\nrm{\bias_{k+1}}^2
\Bigg\}.
\end{align}
Let $\delta \in (0,1)$. Since
\begin{align*}
- \frac{1}{2} \inner{ \nabla f(x_k), \step_{k+1} \noise_{k+1} } & \leq \step_{k+1}^{2-\delta}\nrm{  \nabla f(x_k)}^2 + \step_{k+1}^\delta \nrm{\noise_{k+1}}^2, \\
- \frac{1}{2} \inner{ \nabla f(x_k), \step_{k+1} \bias_{k+1} } & \leq \step_{k+1}^{2}\nrm{  \nabla f(x_k)}^2 +  \nrm{\bias_{k+1}}^2,
\end{align*}we have
\begin{align}
&e^{-\frac{1}{2}f(x_k)} \E e^{\frac{1}{2}f(x_{k+1})}\\
    &\hspace{4mm}\leq \E\exp\Bigg\{
     \frac{1}{2} \inner{\nabla f(x_k), \tilde{x}_{k+1}-x_k} + \frac{L}{2} \nrm{\tilde{x}_{k+1} -x_k}^2  \\
       &\hspace{10mm}     +\prn*{\step_{k+1}^{2-\delta}+\step_{k+1}^{2} }\nrm{\nabla f(x_k)}^2  + \prn*{L\step_{k+1}^2 + \step_{k+1}^\delta}\nrm{\noise_{k+1}}^2+ \prn*{L\step_{k+1}^2+1}\nrm{\bias_{k+1}}^2
\Bigg\}.
\end{align}Invoking \eqref{eq:bias-condition} an denoting $c' \triangleq \prn*{L\step_{k+1}^2+1}\cdot c$, we get
\begin{align}
e^{-\frac{1}{2}f(x_k)} \E e^{\frac{1}{2}f(x_{k+1})} &\leq e^{A_k} \cdot \E\exp\Bigg\{
     \frac{1}{2} \inner{\nabla f(x_k), \tilde{x}_{k+1}-x_k} + \frac{L}{2} \nrm{\tilde{x}_{k+1} -x_k}^2   + c' \cdot \step_{k+1}\nrm{\xi_{k+1}}^2
\Bigg\},
\end{align}
where, 
\begin{equation}
    \begin{aligned}
A_k \triangleq& \prn*{\step_{k+1}^{2-\delta}+\step_{k+1}^{2} + c'\step_{k+1}^{2}  }\nrm{\nabla f(x_k)}^2  \\
    &+ \prn*{L\step_{k+1}^2 + \step_{k+1}^\delta}\nrm{\noise_{k+1}}^2 \\
    &+ c'\prn*{\step^2_{k+1}\nrm{\noise'_{k+1}}^2  +  \step_{k+1} \nrm{\xi_{k+1}'}^2}.
    \end{aligned}
\end{equation}
Recalling that $\sqrt{2\step_{k+1}}\xi_{k+1} = \tilde{x}_{k+1}-x_k + \step_{k+1}\nabla f(x_k)$, we have $  \step_{k+1}\nrm{\xi_{k+1}}^2 \leq  \nrm{\tilde{x}_{k+1}-x_k}^2 + \step_{k+1}^2\nrm{\nabla f(x_k)}^2 $, and thus
\begin{align}
e^{-\frac{1}{2}f(x_k)} \E e^{\frac{1}{2}f(x_{k+1})} &\leq e^{A'_k} \cdot \E\exp\Bigg\{
     \frac{1}{2} \inner{\nabla f(x_k), \tilde{x}_{k+1}-x_k} + \prn*{\frac{L}{2}+c'} \nrm{\tilde{x}_{k+1} -x_k}^2   
\Bigg\},
\end{align}
where $A'_k = A_k + c'\step_{k+1}^2\nrm{\nabla f(x_k)}^2$. \cref{lem:helper} then implies
\begin{align}
e^{-\frac{1}{2}f(x_k)} \E e^{\frac{1}{2}f(x_{k+1})} &\leq e^{A''_k} \cdot (1- \step_{k+1}L')^{-\frac{d}{2}}
\end{align}where $A''_k = A'_k - \frac{\step_{k+1}}{4}\nrm{\nabla f(x_k)}^2$.

We now take the expectation over $x_k, \noise_{k+1}, \noise'_{k+1}, \xi'_{k+1}$ (in other words, we are now only conditioning on $x_k$). Set $\eps \triangleq (1- \step_{k+1}L')^{-\frac{1}{2}} -1 > 0$. Since $\noise_{k+1}, \noise'_{k+1}, \xi'_{k+1}$ are sub-Gaussian and since $\step_k \to 0$, for $k$ sufficiently large we have 
\begin{align}
\E A''_k &\leq (1+\eps)\cdot \exp\brk*{\prn*{ - \frac{\step_{k+1}}{4}  + \step_{k+1}^{2-\delta}+\step_{k+1}^{2} + c'\step_{k+1}^{2}   + c'\step_{k+1}^2}\nrm{\nabla f(x_k)}^2 } \\
&\leq (1+\eps)\cdot e^{ -\frac{\step_{k+1}}{8} \nrm{\nabla f(x_k)}^2}.
\end{align}To summarize, we have shown that, conditioned on $x_k$, 
\begin{equation}
e^{-\frac{1}{2}f(x_k)} \E e^{\frac{1}{2}f(x_{k+1})} \leq (1- \step_{k+1}L')^{-\frac{d+1}{2}} e^{ -\frac{\step_{k+1}}{8} \nrm{\nabla f(x_k)}^2}.
\end{equation}
A simple induction \`a la \cite[Lemma 1 \& Proposition 8]{durmus2017nonasymptotic} then concludes the proof.\hfill\qed

\subsection{Proof of \texorpdfstring{\cref{thm:LRMstability2}}{Theorem 4} for Mirror Langevin}%
\label{sec:mirror-weak-dissipativity}

Here, we bring the proof of \cref{thm:LRMstability2} for the case of \cref{ex:Mirror-LD} and without
noise. The proof for the noisy case is the same as in \cref{sec:proof-weak-dissipativity-ld}.

Define
\[
  x^+ = x - \step \grad f \circ \grad \phi^*(x) + \sqrt{2\step} (\grad^2 \phi^*(x)^{-1})^{1/2} \xi,
\]
where $\xi$ is a standard Gaussian random variable. Let $U(x) = f(\grad \phi^*(x))$. For a fixed
$x$, we have
\[
  \E e^{\frac{1}{2}U(x^+) - \frac{1}{2}U(x)} = \frac{1}{(2\pi)^{d/2}} \int
  \exp\prn*{\frac{1}{2}U(x^+) - \frac{1}{2}U(x) - \frac{\nrm{\xi}^2}{2}}\,d\xi
\]
Notice that we have
\[
  \xi = \frac{1}{\sqrt{2\step}} (\grad^2 \phi^*(x))^{1/2} \prn*{x^+ - x + \step
  \grad f \circ \grad \phi^*(x)}
\]
which implies 
\[
  d\xi = (\sqrt{2\step})^{-d}\sqrt{\det \grad^2 \phi^*(x)}\,dx^+
\]
Thus, the integral, after the change of variable from $\xi$ to $x^+$ becomes
\begin{equation}\label{eq:integral-intermsof-xplus}
  \frac{1}{C}\int\exp\prn*{\frac{1}{2} U(x^+) - \frac{1}{2} U(x) - \frac{1}{4\step} \nrm{(\grad^2 \phi^*(x))^{1/2} \prn*{x^+ - x + \step
  \grad f \circ \grad \phi^*(x)}}^2}\,dx^+
\end{equation}
with $C = (4\pi\step)^{d/2}\sqrt{\det \grad^2 \phi^*(x)^{-1}}$. Now we use the smoothness
of $f$:
\begin{align*}
  U(x^+) - U(x) &= f(\grad \phi^*(x^+)) - f(\grad \phi^*(x)) \\
  &\leq {\color{Maroon}\inner{\grad^2 \phi^*(x) \grad f (\grad \phi^*(x)), x^+ - x }} + \frac{L}{2} \nrm{x^+ -
  x}^2
\end{align*}
On the other hand, we have
\begin{align*}
  &\nrm{(\grad^2 \phi^*(x))^{1/2} \prn*{x^+ - x + \step \grad f \circ \grad \phi^*(x)}}^2 \\
  &\qquad = \nrm{(\grad^2 \phi^*(x))^{1/2}(x^+ - x)}^2
  + \step^2 \nrm{(\grad^2 \phi^*(x))^{1/2}\grad f(\grad \phi^*(x))}^2 \\
  &\qquad\qquad + {\color{Maroon}2\step \inner{\grad^2 \phi^*(x) \grad f \grad \phi^*(x), x^+ - x}}
\end{align*}
Notice that in \eqref{eq:integral-intermsof-xplus}, the colored terms cancel out, and what we are
left with is
\begin{align*}
  &\E e^{\frac{1}{2}U(x^+) - \frac{1}{2}U(x)} \\
  &\leq \frac{1}{C} \int \exp \prn*{
    \frac{L}{4}\nrm{x^+ - x}^2 - \frac{1}{4\step}\nrm{(\grad^2 \phi^*(x))^{1/2}(x^+ - x)}^2
    - \frac{\step}{4}\nrm{(\grad^2 \phi^*(x))^{1/2}\grad f(\grad \phi^*(x))}^2
  }\,dx^+
\end{align*}
As, by our assumption, $\grad^2 \phi^*$ is bounded from above and below, we get the exact form as in
\cref{lem:helper}. The rest of the proof is the same as in \cref{sec:proof-weak-dissipativity-ld}. \hfill\qed

\subsection{Proof of \texorpdfstring{\cref{prop:example-proofs}}{Proposition 1}}

In this section, we prove that \crefrange{ex:RMM}{ex:PLA} satisfy our bias conditions, which, as
we have seen in \cref{sec:stability}, implies \cref{prop:example-proofs}. For brevity, we write
$\Fil_k$ for $\Fil_{\tau_k}$.

\para{Proof for \cref{ex:RMM}}
For randomized mid-point method, by replacing $\stochgrad f(x_k)$ and $\stochgrad f(\xmid{k})$ with
$\grad f(x_k) + \noise_{k+1}'$ and $\grad f(\xmid{k}) + \noise_{k+1}$ respectively, we have
\begin{align*}
  \xmid{k} &= x_k - \step_{k+1}\alpha_{k+1}\crl{\grad f(x_k) + \noise_{k+1}'} +
    \sqrt{2\step_{k+1}\alpha_{k+1}} \stdgauss_{k+1}', \\
  x_{k+1} &= x_k - \step_{k+1}\crl{\grad f(\xmid{k}) + \noise_{k+1}} +
    \sqrt{2\step_{k+1}} \stdgauss_{k+1},
\end{align*}
where $\crl{\alpha_k}$ are \iid and uniformly distributed in $[0,1]$, $\crl{\noise_k}$ and
$\crl{\noise_k'}$
are noises in evaluating $\grad f$ at the corresponding points, and $\stdgauss_k, \stdgauss_k'$ are
independent standard Gaussians.

Notice that the Lipschitzness of $\grad f$, and the fact that $\alpha_k \leq 1$ implies that the
bias term $\bias_{k+1} \coloneqq \grad f(\xmid{k}) - \grad f(x_k)$ satisfies
\begin{align*}
  \exof{\nrm{\bias_{k+1}}^2 \given \Fil_k}
  &\leq L^2\exof{\nrm{\xmid{k} - x_k}^2 \given \Fil_k} \\ 
    &\leq L^2 \prn*{ \step_{k+1}^2\exof{\nrm{\grad f(x_k) + \noise_{k+1}'}^2 \given \Fil_k} +
      2\step_{k+1}d } \\
    &\leq 2L^2\step_{k+1}^2\, \nrm{\grad f(x_k)}^2 + 2L^2\step_{k+1}^2 \sigma^2  + 2L^2d\step_{k+1}\\
    &= \bigoh(\step_{k+1}^2\nrm{\grad f(x_k)}^2 + \step_{k+1}).
\end{align*}

\para{Proof for \cref{ex:ORMM}}
Recall that the new algorithm Optimistic Randomized Mid-Point Method has the iterates
\begin{align*}
  \xmid{k} &= \dpt_k - \step_{k+1}\alpha_{k+1}\stochgrad f(x_{k-\frac{1}{2}}) +
    \sqrt{2\step_{k+1}\alpha_{k+1}}\stdgauss_{k+1}', \\
  x_{k+1} &= \dpt_k - \step_{k+1}\stochgrad f(\xmid{k}) +
    \sqrt{2\step_{k+1}}\, \stdgauss_{k+1},
\end{align*}
where $\crl{\alpha_k}$, $\stdgauss_k, \stdgauss_k'$, and $\stochgrad f$ are the same as in
\eqref{eq:randomized-midpoint}, and the noise and bias are
$\noise_{k+1} \coloneqq \stochgrad f(\xmid{k}) - \grad f(\xmid{k})$ and 
$\bias_{k+1} \coloneqq \grad f(\xmid{k}) - \grad f(\dpt_k)$.
We have
\begin{align*}
  \exof{\nrm{\bias_{k+1}}^2 \given \Fil_k} 
  &= \exof{\nrm{\grad f(\xmid{k}) - \grad f(\dpt_k)}^2 \given \Fil_k} \\
  & \leq L^2 \exof{\nrm{\xmid{k} - \dpt_{k}}^2 \given \Fil_k} \\
  & = L^2 \exof{\nrm{-\step_{k+1}\alpha_{k+1}\stochgrad f(x_{k-\frac{1}{2}}) +
    \sqrt{2\step_{k+1}\alpha_{k+1}}\stdgauss_{k+1}'}^2 \given \Fil_k} \\
  & \leq 2L^2\step_{k+1}^2 \exof{\nrm{\grad f(x_{k-\frac{1}{2}})} \given \Fil_k} +
  2L^2\step_{k+1}^2 \sigma^2 + 4L^2d\step_{k+1}.
\end{align*}
Similar to the proof for \cref{ex:PLA}, notice that
$\nrm{\grad f(\dpt_{k- \frac{1}{2}})}^2 \leq 2\nrm{\grad f(\dpt_{k- \frac{1}{2}}) - \grad f(\dpt_{k})}^2 + 2\nrm*{\grad f(\dpt_{k})}^2 $. 
As $\step_k \to 0$, one can assume that $2L^2\step_{k+1}^2 < \frac{1}{2}$, and we get 
\[
  \exof{\nrm{\bias_{k+1}}^2 \given \Fil_k} \leq 4L^2 \step_{k+1}^2 \nrm{\grad f(\dpt_k)}^2 +
  4L^2\step_{k+1}^2 \sigma^2 + 8L^2d\step_{k+1} 
  = \bigoh(\step_{k+1}^2 \nrm{\grad f(\dpt_k)}^2 + \step_{k+1}),
\]
as desired.\hfill\qed

\para{Proof for \cref{ex:SRK-LD}}
The iterates of stochastic Runge-Kutta Langevin algorithm is as follows:
\begin{align*}
  h_1 &= x_k + \sqrt{2\step_{k+1}}\brk*{(1/2 + 1/\sqrt{6})\,\xi_{k+1} + \stdgauss'_{k+1}/\sqrt{12}} \\
  h_2 &= x_k - \step_{k+1}\crl{\grad f(x_k) + U_{k+1}'} + \sqrt{2\step_{k+1}}\brk*{(1/2 - 1/\sqrt{6})\,\xi_{k+1} +
  \stdgauss'_{k+1}/\sqrt{12}} \\
  x_{k+1} &= x_k - \frac{\step_{k+1}}{2}(\grad f(h_1) + \grad f(h_2)) +
  \step_{k+1}U_{k+1} + 
  \sqrt{2\step_{k+1}}\,\xi_{k+1},
\end{align*}
where $\xi_{k+1}$ and $\stdgauss'_{k+1}$ are independent standard Gaussian random variables independent
of $x_k$, and $U_{k+1}$ and $U_{k+1}'$ are noise in the evaluation of $f$.

Observe that 
\begin{align*}
  \bias_{k+1} = \frac{1}{2}(\grad f(h_1) - \grad f(x_k)) + \frac{1}{2}(\grad f(h_2) - \grad f(x_k)).  
\end{align*}
We have
\[
  \exof{\nrm{\grad f(h_1) - \grad f(x_k)}^2 \given \Fil_k} \leq 2L^2d(1/4 + 1/6 + 1/12)\step_{k+1} 
  = \bigoh(\step_{k+1}),
\]
and
\begin{align*}
  \exof{\nrm{\grad f(h_2) - \grad f(x_k)}^2 \given \Fil_k}
  &\leq 2L^2\prn*{ \step_{k+1}^2 \nrm{\grad f(x_k)}^2 + 2\step_{k+1}^2\sigma^2 + 
  2d (1/4 - 1/6 + 1/12 )\step_{k+1} } \\
  &= \bigoh\prn{\step_{k+1}^2 \nrm{\grad f(x_k)}^2 + \step_{k+1}}.
\end{align*}
We thus have 
\begin{align*}
  \exof{\nrm{\bias_{k+1}}^2 \given \Fil_t} &\leq \frac{1}{2}\exof{\nrm{\grad f(h_1) - \grad
  f(x_k)}^2 \given \Fil_k} +
  \frac{1}{2}\exof{\nrm{\grad f(h_2) - \grad f(x_k)}^2 \given \Fil_k} \\
  &= \bigoh\prn{\step_{k+1}^2 \nrm{\grad f(x_k)}^2 + \step_{k+1}},
\end{align*}
as desired.\hfill\qed

\para{Proof for \cref{ex:Mirror-LD}}
Suppose $\phi$ is a Legendre function \citep{rockafellar2015convex} for $\R^d$, and consider the
iterates
\[
  x_{k+1} = x_k - \step_{k+1} \grad f(\grad \phi^*(x_k)) + \sqrt{2\step_{k+1}}\prn{\grad^2
  \phi^*(x_k)^{-1}}^{1/2}\,\stdgauss_{k+1},
\]
where $\phi^*$ is the \emph{Fenchel dual} of $\phi$, that is, $\phi^*(x) = \sup_{y\in \R^d}
(\inner{x, y} - \phi(y))$. Also recall that \citep{rockafellar2015convex}
\[
  \grad \phi(\grad \phi^*(x)) =x, \quad \grad^2 \phi^*(\grad \phi(x))^{-1} = \grad^2\phi(x), \quad
  \forall x \in \R^d.
\]

Let $v = -\grad f \circ \grad \phi^*$ and $\sigma = \prn{\grad^2 \phi^*}^{-1/2}$.  First, we
mention what our assumptions imply on $f$:
\begin{itemize}
  \item The Lipschitzness of $v$ corresponds to a similar condition in \citep[A2]{li2021mirror}:
    \[
      \nrm{\grad f(x) - \grad f(y)} \leq L \nrm{\grad \phi(x) - \grad \phi(y)}
    \]
  \item The Lipschitzness of $\sigma$ in Frobenius norm corresponds to \emph{modified self-concordance}
    in \citep[A1]{li2021mirror}:
    \[
      \nrm{\grad^2 \phi(x)^{1/2} - \grad^2 \phi(y)^{1/2}}_F \leq L\nrm{ \grad \phi(x) - \grad
      \phi(y)}.
    \]
  \item Boundedness of $\sigma$ in Hilbert-Schmidt norm implies
    \[
      \nrm*{\grad^2 \phi(x)^{-1/2}}_F \leq C_\sigma.
    \]
    
  \item Dissipativity and weak-dissipativity of $v$ corresponds to the conditions below, respectively:
    \[
      \inner{\grad \phi(x), \grad f(x)} \geq \alpha \nrm{\grad \phi(x)}^2 - \beta,\quad
      \inner{\grad \phi(x), \grad f(x)} \geq \alpha \nrm{\grad \phi(x)}^{1 + \kappa} - \beta.
    \]
\end{itemize}
If $f$ and $\phi$ satisfy the conditions above, then the mirror Langevin algorithm
\cref{ex:Mirror-LD} fits into the \eqref{eq:standard-template} scheme. 

\begin{remark*}
Note that this version of Mirror Langevin \emph{cannot} handle the case where $e^{-f}$ is supported
on a compact domain; in that case, the Hessian of $\phi$ \emph{has to} blow up near the boundary,
and will not satisfy our boundedness assumption. The version of mirror Langevin we consider in this
paper, though, can be thought as an adaptive conditioning method for densities supported on $\R^d$.
This setting has also been studied in the literature, see \citep{tzen2023}.
\end{remark*}
\para{Proof for \cref{ex:PLA}}
The iterates of \eqref{eq:pla} follow
\begin{equation}\tag{\ref{eq:pla}}
  \dpt_{k+1} = \dpt_k - \step_{k+1} \grad f(\dpt_{k+1}) + \sqrt{2\step_{k+1}}\,\stdgauss_{k+1}.
\end{equation}
We mentioned that the bias term is $\bias_{k+1} = \grad f(\dpt_{k+1}) - \grad f(\dpt_k)$. Now it
remains to prove that it satisfies the conditions \eqref{eq:bias-assumption} and
\eqref{eq:bias-condition}. 
We have 
\begin{align*}
  \exof{\nrm{\bias_{k+1}}^2 \given \Fil_k} 
  &= \exof{\nrm*{\grad f(\dpt_{k+1}) - \grad f(\dpt_k)}^2 \given \Fil_k} \\
  & \leq L^2 \exof{\nrm{\dpt_{k+1} - \dpt_{k}}^2 \given \Fil_k} \\
  & = L^2 \exof{\nrm{-\step_{k+1}\grad f(\dpt_{k+1}) + \sqrt{2\step_{k+1}}\,\stdgauss_{k+1}}^2
  \given \Fil_k} \\
  & \leq 2L^2 \step_{k+1}^2 \exof{\nrm*{\grad f(\dpt_{k+1})}^2 \given \Fil_k} + 4L^2d\step_{k+1}.
\end{align*}
Now, notice that $\nrm*{\grad f(\dpt_{k+1})}^2 \leq 2\nrm*{\grad f(\dpt_{k+1}) - \grad f(\dpt_{k})}^2
+ 2\nrm*{\grad f(\dpt_{k})}^2 $. As $\step_k \to 0$, one can assume that $2L^2\step_{k+1}^2 <
\frac{1}{2}$, and we get
\[
  \exof{\nrm{\bias_{k+1}}^2 \given \Fil_k} \leq \frac{1}{2} \exof{\nrm{\bias_{k+1}}^2 \given \Fil_k} 
  + \nrm{\grad f(\dpt_k)}^2 + 4L^2d\step_{k+1},
\]
which implies
\[
  \exof{\nrm{\bias_{k+1}}^2 \given \Fil_k} \leq 4L^2 \step_{k+1}^2 \nrm{\grad f(\dpt_k)}^2 + 8L^2d\step_{k+1} 
  = \bigoh(\step_{k+1}^2 \nrm{\grad f(\dpt_k)}^2 + \step_{k+1}),
\]
as desired.\hfill\qed

\end{document}